\documentclass{article}

\PassOptionsToPackage{numbers, compress}{natbib}
\usepackage[preprint]{neurips_2026}
\usepackage[utf8]{inputenc} % allow utf-8 input
\usepackage[T1]{fontenc}    % use 8-bit T1 fonts
\usepackage{hyperref}       % hyperlinks
\usepackage{url}            % simple URL typesetting
\usepackage{booktabs}       % professional-quality tables
\usepackage{amsfonts}       % blackboard math symbols
\usepackage{nicefrac}       % compact symbols for 1/2, etc.
\usepackage{microtype}      % microtypography
\usepackage{xcolor}         % colors

\usepackage{mathtools} % amsmath with fixes and additions
\usepackage{booktabs} % commands to create good-looking tables
\usepackage{tikz} % nice language for creating drawings and diagrams
\usepackage{amssymb}
\usepackage{subcaption}
\usepackage{graphicx}
\usepackage{appendix}
\usepackage{subcaption}
\usepackage{cleveref}
\usepackage{multirow}
\usepackage{comment}
\usepackage{algorithm}
\usepackage{algpseudocode}
\usepackage[multiple]{footmisc}
\usepackage{bm}
\usepackage[british]{babel}
\usepackage{enumitem}
\DeclareMathOperator*{\argmax}{argmax}

\usepackage{amsthm}
\theoremstyle{definition}
\newtheorem{defn}{Definition}[section]
\theoremstyle{proposition}
\newtheorem{prop}{Proposition}[section]
\theoremstyle{lemma}

\theoremstyle{plain}
\newtheorem{thm}{Theorem}[section]
\usepackage{wrapfig}

\usepackage{array}   % for \newcolumntype macro
\newcolumntype{L}{>{$}l<{$}} % math-mode version of "l" column type

\usepackage[most]{tcolorbox}

\newtcolorbox{axiombox}{
  colback=gray!5,
  colframe=black,
  boxrule=0.5pt,
  arc=2pt,
  left=6pt,
  right=6pt,
  top=6pt,
  bottom=6pt
}

\Crefname{figure}{Fig.}{Figs.}

\title{On the QUEST for Uncertainty Quantification via\\Highest Density Regions}% methods for regression}

\author{
  Sam Goring\\
  King's College London\\
  \texttt{sam.goring@kcl.ac.uk}
  \And
  Tom Kuipers \\
  Northeastern University London \\
  \texttt{t.kuipers@northeastern.edu}
  \AND
  Nicola Paoletti \\
  King's College London\\
  \texttt{nicola.paoletti@kcl.ac.uk}
  \And
  David S. Watson \\
  King's College London\\
  \texttt{david.watson@kcl.ac.uk}
}

\begin{document}

\maketitle

\begin{abstract}
Uncertainty quantification (UQ) is essential for reliable decision‑making in safety‑critical applications in probabilistic machine learning.
For regression problems, dominant scalar UQ approaches---notably, those based on proper scoring rules---measure uncertainty via pointwise predictive risk.
This can lead to counterintuitive results when the target statistic is not the conditional expectation.
We propose an alternative framework, in which uncertainty is characterised by the volume of the most probable subset of a distribution's support.
QUEST (Quantifying Uncertainty via highest dEnSiTy regions) is a novel approach to UQ based on the concentration of Lebesgue measure at a distribution's peak(s), evaluated at one or more values of a robustness parameter $\alpha$.
We establish connections between our measures and classical statistics from information theory and economics.
We show that, unlike popular alternatives based on proper scoring rules, QUEST measures of epistemic and aleatoric uncertainty satisfy a set of axioms adapted from the UQ literature, including monotonicity under distributional spread and invariance to location shifts. 
Selective prediction benchmarks confirm that QUEST performs favourably against standard measures such as variance and differential entropy.
\end{abstract}
\section{Introduction}
Uncertainty quantification (UQ) is crucial for making reliable decisions in safety-critical settings with probabilistic machine learning techniques \citep{hullermeier_aleatoric_2021, pml1Book}.
In scalar-valued UQ, entropy-based measures have long been the uncertainty measure (UM) of choice for classification tasks \citep{houlsby_bayesian_2011, kendall_what_2017, charpentier_disentangling_2022}, although that consensus has been challenged by recent works \citep{wimmer2023, schweighofer_introducing_2023}.
These analyses have been extended to regression settings \citep{bulte_axiomatic_2025}, where proper scoring rules (PSRs) \citep{gneiting_strictly_2007} provide a unified conceptual approach \citep{kotelevskii_risk_2025, bulte_uncertainty_2025}.
These methods operationalise UQ as the pointwise risk of predicting with a particular model as opposed to the true distribution \citep{fishkov_uncertainty_2025}. 

Despite their theoretical appeal, PSR-based measures can lead to counterintuitive results, particularly when the target statistic is not the conditional expectation. 
Consider the task of estimating the conditional \textit{mode}, better known as the maximum a posteriori (MAP) value. 
MAP inference arises naturally in many settings, e.g. reinforcement learning and Bayesian modelling. When distributions are unimodal and symmetric, means and modes coincide and standard regression methods can be safely deployed. 
However, means and modes can diverge arbitrarily in the presence of outliers, multimodality, or heavily skewed data. 
In these settings, conditional expectations may lie in a region of low probability density, meaning standard regressors will make predictions that are unlikely or even impossible (see \Cref{fig:splash}). 
Under such conditions, the MAP is arguably a more informative and useful statistical target than the conditional expectation. 

How should we characterise the uncertainty of MAP predictions? Using a standard plug-in estimator like variance could substantially overstate the uncertainty in some cases, as is evident in \Cref{fig:splash}. This reflects a more general shortcoming with many common UM methods, which fail to properly prioritise concentration of measure. We start from the basic premise that a valid UM ought to track this fundamental property, with uncertainty increasing as a monotonic function of a distribution's spread. 
This volume-centric perspective suggests a novel framework for UQ.

\begin{figure}
    \centering
    \includegraphics[width=0.9\linewidth]{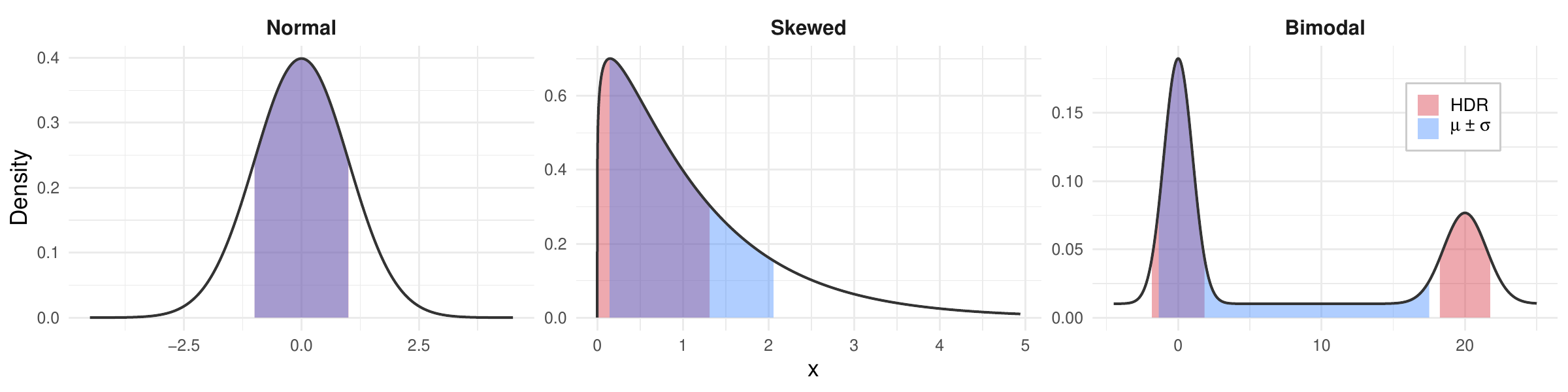}
    \caption{Comparison of confidence sets for the mean (blue, defined as the interval $\mu\pm \sigma$) and the mode (red, defined as the $(1-\alpha)$-high density region, or HDR). 
    % Confidence sets for means and modes can vary widely. 
    The two coincide for symmetric, unimodal distributions such as the standard \textbf{Normal}. 
    In \textbf{Skewed} data, however, variance grows due to tail area mass.
    In \textbf{Bimodal} distributions, the HDR may be disconnected, while confidence intervals can assign positive weight to low-density regions.
    For all three settings, we fix $\alpha = 2 \Phi(-1) \approx 0.317$, where $\Phi$ is the standard normal CDF, so that the HDR coincides with $\mu \pm \sigma$ in the left panel.}
    \label{fig:splash}
\end{figure}

We introduce QUEST, a family of UMs based on the Lebesgue measure of a distribution's highest density region. 
QUEST aligns with intuition in the motivating examples above and offers a flexible, interpretable alternative to PSR-based approaches.
We develop measures of aleatoric, epistemic, and total uncertainty within this framework.
Our primary contributions are: (1) \textbf{Conceptual: }We propose an interpretation of uncertainty based on density concentration and a user-supplied robustness parameter $\alpha$. (2) \textbf{Methodological: }We introduce novel measures of local and global uncertainty in terms of highest density regions, and establish their relationship to well-studied statistics from information theory and econometrics. (3) \textbf{Theoretical: } We show that QUEST measures of epistemic and aleatoric uncertainty are sound w.r.t. a set of UM axioms adapted from previous works \citep{wimmer2023, sale_second-order-dist_2023, bulte_axiomatic_2025}. (4) \textbf{Empirical: }Selective prediction benchmarks confirm that QUEST performs favourably against variance and differential entropy, the dominant UM methods in use today for regression. 

\section{Related work}
\paragraph{Uncertainty quantification in regression}
Many authors have developed frameworks for representing uncertainty in regression via second-order distributions \citep{malinin_regression_2020, meinert_multivariate_2022} (formally defined in \Cref{sec:uncert-rep}) or ensembles \citep{gal_dropout, lakshminarayanan_simple_2017, berry_epistemic_2026}.
Common summary statistics for measuring the uncertainty encoded by such distributions include variance and differential entropy. 

However, unlike in the classification setting, where many authors have investigated the empirical and theoretical performance of UMs \citep{houlsby_bayesian_2011, kendall_what_2017, wimmer2023, sale_second-order-dist_2023}, regression uncertainty has received relatively little attention. 
Theoretical analysis, in the form of axiomatic assessment, extends existing frameworks from the classification literature \citep{wimmer2023, sale_second-order-dist_2023} to the regression setting \citep{bulte_axiomatic_2025, bulte_uncertainty_2025}.
Empirical validation typically involves downstream tasks, such as selective prediction \citep{fishkov_uncertainty_2025}; active learning \citep{fishkov_uncertainty_2025, berry_epistemic_2026}; out of distribution \citep{fishkov_uncertainty_2025, schweighofer_introducing_2023, amini_deep_2020, malinin_regression_2020} and adversarial sample detection \citep{schweighofer_introducing_2023, amini_deep_2020};
low correlation between UMs \citep{fishkov_uncertainty_2025};
and alignment with expected behaviour on toy datasets \citep{lakshminarayanan_simple_2017,meinert_multivariate_2022, valdenegro-toro_deeper_2022, malinin_regression_2020, berry_epistemic_2026}.

\paragraph{Proper scoring rules}
Existing approaches to UQ in regression have been shown to fit a unified framework when viewed as optimality targets for proper scoring rules (PSRs) \citep{fishkov_uncertainty_2025, bulte_uncertainty_2025}.
Introduced by \citet{gneiting_strictly_2007}, PSRs measure uncertainty as the expected pointwise risk of predicting with the learned distribution.
Within the PSR framework, the choice of loss function informs the relevant UM.
Optimising the \textit{Brier score}---the PSR corresponding to squared error loss---yields variance as the optimal expected score. However, the Brier score is proper only for the mean functional. By contrast, the \textit{log score} is strictly proper for the full distribution, and its optimal expected value recovers differential entropy (formally defined below). Variance and differential entropy thus emerge as the natural UMs under squared error and log loss, respectively.

Though loss functions from robust statistics (e.g., Huber loss \citep{huber2009} or the maximum correntropy criterion \citep{feng2015}) may occasionally converge on the MAP, this is not generally guaranteed. On the contrary, \citet{heinrich2014} has shown that the mode functional is not \textit{elicitable}---in other words, there is no loss or scoring function under which the mode is the Bayes predictor.
Thus the PSR framework is uniquely ill-suited to UQ in MAP problems.

\paragraph{Set-based UQ}
The use of sets to quantify uncertainty has been explored through conformal prediction (CP) and credal sets.
CP provides finite sample probabilistic guarantees on prediction correctness, via a predictive set \citep{vovk_algorithmic_2005}.
CP does not explicitly disentangle aleatoric and epistemic uncertainties \citep{sale_aleatoric_2025}, and its ability to measure the former has been questioned \citep{hagos_performance_2026}.
Credal sets utilise an alternative uncertainty representation based on imprecise probability theory \citep{walley_statistical_1991}.
UMs of credal classifiers have been developed based on volume- \citep{sale_is_2023} and distance-based measures \citep{gonzalez-garcia_quantification_2026}; PSRs \citep{hofman_quantifying_2024}; as well as entropy and the generalised Hartley measure \citep{wang_credal_2024}.
Works at the intersection of these approaches provide methods for constructing credal sets \citep{caprio_conformalized_2024, javanmardi_conformalized_2024} and explore links between CP and imprecise high density regions (IHDRs) \citep{caprio_conformal_2025}.
\citet{caprio_conformal_2025} find that a classic transductive CP region at level $\alpha$ is the IHDR at level $\alpha$. 
Though similarly rooted in volume-centric measures, this work is situated within both the imprecise probability and set-valued uncertainty measurement traditions.
In contrast, though utilising a set-based representation, QUEST does not rely on imprecise probabilities and provides scalar-valued UMs.

% \paragraph{Summary} Existing approaches to UQ in regression either rely on risk-based scalar measures tied to mean prediction, or on set-based frameworks oriented towards coverage and imprecise probabilities. While HDRs are well-studied in statistics, they have not been operationalised as uncertainty measures for predictive tasks. Moreover, to our knowledge, no existing work systematically examines UQ for MAP inference. QUEST fills this gap by providing scalar valued MAP aligned UMs via HDRs within a precise probability second order uncertainty representation.

\section{Preliminaries}
\label{sec:uncert-rep}
\paragraph{Notation}
Let $\mathcal X \subseteq \mathbb R^d$ be a continuous feature space.
We denote by $p_\theta(x)$ the probability density function (pdf) evaluated at $X=x$ under the distribution parametrised by $\theta \in \Theta \subseteq \mathbb{R}^k$ (though we occasionally suppress the dependence on $\theta$ for notational convenience).
We assume the existence of some $\theta^*$ such that the true density is $p_{\theta^*}(x)$.
For example, if $X$ is Gaussian, then $\theta^* = \big(\mu^*, \sigma^*\big)$ would represent the true mean and standard deviation.
We write $P(X \in A)$ as shorthand for $\int_{x \in A} ~p(x) ~dx$, where $A \subseteq \mathcal X$.
We denote a Dirac mass at $X=x$ via $\delta_x$.

\paragraph{Highest density regions}
Let $p$ be a pdf with support $\mathcal X \subseteq \mathbb R^d$. 
Then the distribution's $(1 - \alpha)$-\textit{highest density region} (HDR) is defined as $C_\alpha := \{x \in \mathcal X: p(x) \geq t_\alpha(p)\}$,
where $t_\alpha(p)$ is the largest constant such that  $P(X \in C_\alpha) \geq 1 - \alpha$ \citep{hyndman_computing_1996}.
Such a set is maximally informative w.r.t. $p$, providing the tightest possible localisation of $X$ for a fixed confidence level \citep{brown_optimal_1995}.
HDRs are volume-minimizing among measurable sets with fixed mass, and are unique up to null sets when the threshold level set has zero Lebesgue measure.
They come with strong coverage guarantees via the Ghosh-Pratt lemma \citep{ghosh_relation_1961, pratt_shorter_1963}.
We observe that the HDR must contain the true mode for any $\alpha < 1$.

\paragraph{Second-order distributions}
In line with related works \citep{wimmer2023, bulte_axiomatic_2025} we assume a bi-level uncertainty representation in which an agent learns a distribution \emph{over} distributions.
Specifically, the learner assigns a degree of belief to each possible $\theta \in \Theta$.
% \footnote{equivalently a frequentist obtains the sampling distribution of $\theta|x$} given training data $\mathcal{D}$.
These beliefs are represented by a second-order distribution $Q \in \mathcal{Q}(\Theta)$ with pdf $q$, where $\mathcal{Q}(\Theta)$ is the space of probability distributions over $\Theta$. 
To continue with our Gaussian example above, $q$ would represent the uncertainty of our estimators for $\mu^*$ and $\sigma^*$.
We assume that (i) $\theta^* \in \Theta$; and (ii) our method for estimating $\theta^*$ is \textit{consistent} (i.e., guaranteed to converge on the true value w.p. $1$). 
% Under these conditions, $q$ asymptotically approaches $\delta_{\theta^*}$ as sample size increases.
% The consistency assumption is in line with previous works in the literature, and reflects the focus of this work on measures of distributions as opposed to the procedure for learning them.
\paragraph{Uncertainty typology}
%Our framework aligns with a subjective Bayesian epistemology in which we adopt the uncertainty representation seen in \cref{sec:uncert-rep} where predictive distributions arise from a second-order distribution $Q$ over parameters $\theta$, together with a consistent estimator of $\theta^*$. 
% In this work, measures reflecting sources of uncertainty have the following meaning.
\textit{Aleatoric uncertainty} $(AU)$ is the uncertainty inherent to a given data generating process (DGP), which cannot be reduced through sampling or learning.
% We will occasionally distinguish between \textit{true} $AU$, denoted by $AU^*$, and \textit{estimated} $AU$, denoted by $\widehat{AU}$. The latter represents our finite sample estimate of $AU^*$, which will generally be subject to additional noise and error. 
%obtained by measuring the uncertainty in the distribution obtained from $Q$, which we consider to be the best approximation of the true DGP (note this model may differ from our predictive model). 
\textit{Epistemic uncertainty} $(EU)$ is the uncertainty that is reducible in principle, ultimately going to zero in the limit of infinite data under assumptions (i) and (ii).
% For the Gaussian example, a common choice would be to say that $AU$ is determined by $\sigma^*(x)$, while $EU$ is a function of the estimators $\hat \mu, \hat \sigma$.
% For the Gaussian example, a common choice would be to identify $AU^*$ with the true function $\sigma^*(x)$; $\widehat{AU}$ with the estimated function $\hat \sigma(x)$; and $EU$ with the error induced by the estimator $\hat \mu(x)$.
% Such an interpretation is similar to the PSR interpretation of total, Bayes and excess risk \citep{fishkov_uncertainty_2025}. 
\textit{Total uncertainty} $(TU)$ represents a combination of $AU$ and $EU$, possibly but not necessarily in the form of a simple sum \citep{wimmer2023}. 

\paragraph{Information theory}
We briefly review definitions for several fundamental information theoretic quantities \citep{cover2006elements, PolyanskiyWu2025}.
The \textit{differential entropy} of a continuous variable $X$ is defined as $h(p) := - \int_{\mathcal X} p(x) \log p(x) ~dx$.
% For discrete $X$, the \textit{Shannon entropy} is defined as $H(X) := -\sum_x p(x) \log p(x)$. We write $H_2(p)$ for the special case where $X \in \{0,1\}$ (binary entropy), with $P(X=1)=p$.  \textit{Differential entropy} is the continuous analogue, $h(p) := - \int_{\mathcal X} p(x) \log p(x) ~dx$.
When comparing distributions, we often use the \textit{cross-entropy} $h(p, q) := -\int_\mathcal X p(x) \log q(x)~dx$, which is equivalent to $h(p)$ plus the KL-divergence between $p$ and $q$. 
The latter is a famous example of an $f$-\textit{divergence}, a family of distributional dissimilarity measures of the form
 \begin{align*}
     D_f(p ~||~ q) := \int_{\mathcal{X}} f\!\left(\frac{p(x)}{q(x)}\right) q(x)\, dx,
 \end{align*}
where $f: (0, \infty) \to \mathbb{R}$ is a convex function with $f(1) = 0$. Setting $f(t) = t \log t$ recovers the KL-divergence; other choices give the total variation distance ($f(t) = |t-1|/2$), Hellinger distance ($f(t) = (\sqrt{t} - 1)^2$), and $\chi^2$-divergence ($f(t) = (t-1)^2$), among others.

\section{QUEST Measures}
\label{sec:method}
Our guiding intuition is that statistical \textit{confidence} amounts to placing large portions of probability mass in a relatively small region, while statistical \textit{uncertainty} amounts to dispersing that mass more widely across the distribution's support. 
So stated, the claim may seem uncontroversial---yet we show that popular risk-based measures do not always respect this basic requirement.
We begin by introducing two related measures that quantify uncertainty via HDRs: a local measure, which we call \emph{$\alpha$-volume}, and a global measure, which we call \emph{integrated volume}.  
% Both reflect our conception of uncertainty as the concentration of probability density, with $\alpha$-volume capturing concentration at a fixed tolerance level $\alpha$, and integrated volume capturing concentration across all such levels. 
% We establish illuminating connections between our proposed measures and several well-studied statistics from information theory and econometrics.

\paragraph{$\alpha$-volume}
For pdf $p$ with support $\mathcal X \subseteq \mathbb R^d$ and $(1-\alpha)$-HDR $C_\alpha$, the $\alpha$\textit{-volume} is defined as $V_\alpha(p) := \lambda(C_\alpha)$,
where $\lambda$ is the $d$-dimensional Lebesgue measure.
Small values indicate high concentration of density (low uncertainty), while large values indicate low concentration (high uncertainty).
By convention, when $C_\alpha=\emptyset$ (e.g. when $p=\delta_x$), we set $V_\alpha(p)=0$.

This statistic is closely related to differential entropy, which is often proposed as a natural information theoretic UM \citep{cover2006elements, schweighofer2025, fishkov_uncertainty_2025}. However, this measure suffers from several drawbacks. First, unlike its discrete counterpart, differential entropy can be negative. This is a simple consequence of the fact that, whereas probability \textit{mass} is necessarily bounded on the unit interval, probability \textit{density} can exceed $1$. Since many authors take it as axiomatic that any valid UM must be non-negative \citep{wimmer2023, sale_second-order-dist_2023, bulte_axiomatic_2025}, this rules out $h$ directly. Second, differential entropy is sensitive to tail area fluctuations in ways that may not reflect meaningful uncertainty. 
A density with light tails and a density with heavy tails can look identical near the mode yet have very different entropies.

$\alpha$-volume resolves both shortcomings. The precise relationship between our proposed measure and differential entropy is given by the following proposition. (All proofs in \Cref{sec:proofs}.)
\begin{prop}\label{prop:de}
    For any pdf $p$ on $\mathcal X \subseteq \mathbb{R}^d$  and any $\alpha \in (0, 1)$ such that $V_\alpha(p) < \infty$:
    \begin{align*}
        % h(p) = (1-\alpha)\big[\log V_\alpha(p) - D_{\mathrm{KL}}(p(\cdot \mid X \in C_\alpha) \,\|\, u_{\alpha})\big] + \alpha \cdot h(p \mid X \not\in C_\alpha) + H_2(\alpha),
        \log V_\alpha(p) = h(p_\alpha, u_\alpha),
    \end{align*}
    where $p_\alpha := p(\cdot \mid X \in C_\alpha)$ and $u_{\alpha}$ is the uniform density on  $C_\alpha$.
\end{prop}
Thus when density is flat within the HDR, $\alpha$-volume reduces to a trimmed estimator of an exponentiated entropy functional. Exponentiating guarantees nonnegativity, while trimming makes the statistic robust to tail area behaviour. 

\paragraph{Integrated volume}
In settings where the choice of $\alpha$ is unclear, or where the degree of concentration 
across a range of confidence levels is of interest, we define the integrated volume measure as $IV(p) := \int_{0}^1 ~V_\alpha(p) ~d\alpha$.
This quantity summarises the entire $\alpha \mapsto V_\alpha(p)$ function---what we call \textit{the coverage curve}---and provides a global assessment of concentration. 
(For a visual example, see \Cref{sec:app-levelling}.)
Curves that quickly approach zero as $\alpha$ grows suggest more densely concentrated (and therefore less uncertain) distributions.
We remark that integrated volume is finite for any distribution with bounded support, as well as those with unbounded support and tails that decay at an appropriate rate (e.g., subexponential). 
%exponential, subexponential, or even polynomial tails of index $<1$.
For heavier-tailed distributions with infinite support, we can trim the lower limit of the integral with some $\varepsilon > 0$.

% \paragraph{Relation to econometric quantities}
The coverage curve and summary statistic $IV(p)$ are closely related to the \textit{Lorenz curve} and \textit{Gini coefficient} (respectively), two common tools for measuring income inequality \citep{sen1997economic, cowell2011measuring}. 
% With bounded support, the coverage curve is geometrically equivalent to the Lorenz curve, while integrated volume (the area under this curve) is a linear function of the corresponding Gini coefficient.
The Lorenz curve $L : [0, 1] \to [0, 1]$ plots the cumulative share of total income held by the bottom $z$-fraction of a population, ordered from poorest to richest. By construction, $L$ is non-decreasing, convex, and satisfies $L(0) = 0, L(1) = 1$. Perfect equality corresponds to the diagonal $L(z) = z$; greater inequality corresponds to curves that sag below the diagonal. The Gini coefficient is a scalar summary $G := 1 - 2 \int_0^1 L(z)\, dz$,
which measures the area between the Lorenz curve and the diagonal, normalised so that $G = 0$ for perfect equality and $G \to 1$ as inequality grows.

To translate these notions to densities, we treat the density $p$ as an income distribution: the ``population'' is the support $\mathcal X = [0, 1]^d$ with subsets quantified by their Lebesgue measure, and the ``income'' at each location $x$ is the probability density $p(x)$.\footnote{The following Prop. \ref{prop:gini} holds (with minor amendments) for any density with bounded support. The unit volume assumption is simply for convenience.} Sorting locations from lowest to highest density and accumulating mass yields the \emph{Lorenz $p$-curve}:
\begin{equation*}
    L_p(z) := \int_{\{x : p(x) \leq q_p(z)\}} p(x)\, dx, \quad z \in [0, 1],
\end{equation*}
where $q_p(z)$ is the $z$-quantile of the density values under Lebesgue measure on $\mathcal X$. Intuitively, $L_p(z)$ is the total probability mass in the $z$-fraction of the support with the lowest density values. The \emph{Gini $p$-coefficient} $G_p$ is defined as above, plugging in the Lorenz $p$-curve for $L$.
% A uniform $p$ has $L_p(z) = z$ (the diagonal) and $G_p = 0$, indicating no concentration. A highly peaked $p$ has $L_p$ sagging far below the diagonal and $G_p$ close to $1$.

\begin{prop}\label{prop:gini}
    Let $p$ be a density on $[0, 1]^d$. Then we have the following two identities.
    \begin{itemize}[noitemsep]
        \item[(a)] For every $\alpha \in (0, 1)$, the Lorenz $p$-curve satisfies $L_p\big(1 - V_\alpha(p)\big) = \alpha$.
        Equivalently, the curves $\alpha \mapsto V_\alpha(p)$ and $z \mapsto L_p(z)$ are related by the coordinate flip $(\alpha, V_\alpha) \mapsto (1 - V_\alpha, \alpha)$.
        \item[(b)] The Gini $p$-coefficient satisfies $G_p = 1 - 2 \cdot IV(p)$.
    \end{itemize}
\end{prop}

Thus with pdfs on bounded support, the coverage curve is geometrically equivalent to the Lorenz curve, while integrated volume is a linear function of the corresponding Gini coefficient.
Together with Prop. \ref{prop:de}, these results provide a conceptual and a quantitative bridge between our volume-centric approach and classical tools from information theory and economics. 
% These local and global measures, when applied to first- and second-order distributions, provide a unified framework for HDR-based UQ 

% The $\alpha$-$V_\alpha$ curve, plotted with $\alpha$ on the horizontal axis and $V_\alpha(p)$ on the vertical, is geometrically equivalent to the Lorenz curve of the density-as-income distribution: a uniform $p$ traces a 45° diagonal from $(0, 1)$ to $(1, 0)$, while a concentrated $p$ traces a curve that hugs the axes. Integrated volume is the area under this curve, and a linear function of the corresponding Gini coefficient---making $IV(p)$ a direct analogue of standard inequality measures, with low values indicating concentrated densities (high inequality among density values, low uncertainty) and high values indicating diffuse densities (low inequality, high uncertainty).

\paragraph{QUEST Measures}
With these definitions in place, we introduce the QUEST measures of uncertainty (see \Cref{tab:quest_measures}). 
$AU$ is quantified via the $\alpha$-volume and integrated volume of the ground truth distribution $p_{\theta^*}$ for local and global assessments, respectively. This captures the intuition that the irreducible uncertainty of a given DGP is encoded by the concentration of measure on its support.

\begin{table}[t]
    \centering
    \caption{QUEST measures of uncertainty. $TU$ measures require specification of an admissible $f$-divergence and corresponding function $g_f$; see main text.}
    \label{tab:quest_measures}
    \begin{tabular}{lccc}
        \toprule
        & \textbf{Aleatoric} & \textbf{Epistemic} & \textbf{Total} \\
        \midrule
        \textbf{Local}  & $V_\alpha(p_{\theta^*})$ & $V_\alpha(q)$ & $V_\alpha(p_{\theta^*}) \cdot g_f \big( D_f(p^*_\alpha ~||~ \hat p_\alpha) \big)$ \\
        \addlinespace % Adds a small vertical gap for readability
        \textbf{Global} & $IV(p_{\theta^*})$ & $IV(q)$ & $\int_0^1 V_\alpha(p_{\theta^*}) \cdot g_f \big( D_f(p^*_\alpha ~||~ \hat p_\alpha) \big) ~d \alpha$ \\
        \bottomrule
    \end{tabular}
\end{table}

We define $EU$ similarly, this time operating over the second-order pdf $q$. This reflects the inherent subjectivity of $EU$ in our framework. 
This deviates from standard practice, according to which $EU$ is conceived as a sort of residual---what remains after we subtract $AU$ from $TU$. 
Such approaches have no way of describing the epistemic similarity between agents who are equally confident in their predictions, despite incurring different errors. We revisit this point in \Cref{sect:discussion} below.

It would make little sense to simply add $AU$ and $EU$ to obtain $TU$ in our framework, given that $p$ and $q$ are defined over entirely different spaces. Instead, we propose an interaction between first- and second-order uncertainty. 
Given some set of observations $\mathcal D$, define the posterior predictive distribution $\hat p(x) := \int_\theta ~p_\theta(x \mid \theta) ~q(\theta \mid \mathcal D) ~d \theta$.
Let $C^*_\alpha, \hat C_\alpha$ be the $(1 - \alpha)$-HDRs for $p_{\theta^*}$ and $\hat p$, respectively, and define the conditional pdfs:
\begin{align*}
    p^*_\alpha := p_{\theta^*}(\cdot \mid X \in C^*_\alpha), \quad \hat p_\alpha := \hat p(\cdot \mid X \in \hat C_\alpha).
\end{align*}
We say that an $f$-divergence is $TU$-\textit{admissible} if it is bounded above by some $M_f < \infty$ and finite for distributions with arbitrary support overlap (equivalently, $f(0) < \infty$ and $f'(\infty) < \infty$). Let $g_f : [0, M_f) \mapsto [1, \infty)$ be a continuous, strictly increasing function with $g_f(0) = 1$ and $\lim_{D \to M_f} g_f(D) = \infty$. 
Then local and global $TU$ measures are given by the formulae in \Cref{tab:quest_measures}. 
As a concrete example, plugging in the TV distance and defining $g_f(D)=(1-D)^{-1}$ gives $V_\alpha(p_{\theta^*}) / \big(1 - d_{TV}(p^*_\alpha, \hat p_\alpha)\big)$ for local $TU$. For further examples of admissible function-divergence pairs, see \Cref{app:admissible}. 

This measure is motivated by the intuition that $TU$ ought to depend on objective facts ($\theta^*$) and subjective beliefs ($q$) in such a way that it converges to $AU$ when the two align (i.e., when $p^*_\alpha = \hat p_\alpha$), but grows without limit as error increases.
This holds for standard additive decompositions via variance or information-theoretic quantities.
% (For critical perspectives on the additive decomposition of UMs, see \citep{malinin2021uncertainty, wimmer2023}.)
However, our volume-centric approach requires a multiplicative penalty to respect the units of $AU$. 
% More importantly, our subjective interpretation of $EU$ means that $TU$ is not fully determined by aleatoric and epistemic components.

\paragraph{Choosing $\alpha$}
Our framework introduces a hyperparameter $\alpha$ for local UMs. 
Since $\alpha$ controls the strength of the coverage guarantee, it may be tempting to adopt a frequentist perspective and think of this parameter as a sort of type I error rate for the mode. 
This is a valid interpretation in MAP inference settings.
A more general alternative, however, is that $\alpha$ is a \textit{robustness} parameter---it controls the degree of tail area fluctuations that a user deems practically irrelevant for UQ. 
Such trimmed estimators are widely used in robust statistics \citep{huber2009}.
% This interpretation licenses a more permissive approach to setting $\alpha$ than we often find in null hypothesis significance testing, where many disciplines adopt the default $5\%$ more out of habit than conviction.  By contrast, users may have good reason to focus attention on a narrow region of high density, corresponding to relatively large values of $\alpha$. 
Users who are unsure what value to select may find it informative to inspect the coverage curve directly in search of kinks or plateaus.

\section{Axiomatic Assessment}\label{sc:axioms}
We take an axiomatic approach to assessing the QUEST framework. 
Our axioms closely track those of \citet{wimmer2023}, who initiated the axiomatic UQ research program with their paper on measures for classification; \citet{sale_second-order-dist_2023}, who extended and refined the axioms; and \citet{bulte_axiomatic_2025}, who updated the axioms for regression. 

\paragraph{Partial ordering}
A key step in formalising these axioms is establishing a partial order on distributions in terms of relative uncertainty. 
\citet{bulte_axiomatic_2025} do this directly via the variance, which is unsatisfying in at least two respects.
First, variance offers an incomplete picture of distributional uncertainty. Consider the case of the Cauchy distribution, for which the second moment is undefined. It is plainly evident that some Cauchy densities are more concentrated than others (see \Cref{fig:levelling}A), yet variance provides no tools for distinguishing between them.
Second, and perhaps more troubling, this variance-based partial ordering effectively guarantees the conclusion the authors aim to secure---that variance-based UMs satisfy those same axioms. 

To avoid this circularity, we ground our partial ordering in a primitive notion of distributional spread that is independent of any specific uncertainty functional and holds for arbitrary continuous distributions. We begin with the equivalence relation.
% We begin by characterising spread-equivalence.

\begin{defn}[Spread-equivalence]\label{defn:equiv}
    Let $p, p'$ be pdfs on $\mathbb{R}^d$ with supports $\mathcal{X}, \mathcal{X}'$.
    We say that $p'$ is \textit{spread-equivalent} to $p$, written $p \sim p'$, iff:
    \begin{align*}
        \forall t: \lambda\big( \{x: p(x)>t \} \big) = \lambda\big( \{x: p'(x)>t \} \big).
    \end{align*}
\end{defn}
Spread-equivalent distributions are equally concentrated at all density thresholds.
% Examples of spread-preserving transformations include any composition of rigid motions (translations, rotations, reflections). 
To allay concerns that Defn. \ref{defn:equiv} smuggles in our preferred measure of uncertainty, we point out that spread-equivalence can be restated in terms of a natural and widely accepted divergence-theoretic condition: that a distribution's spread should be quantifiable in terms of its deviation from uniformity. The following proposition makes this precise.

\begin{prop}\label{prop:diverge}
    Let $p, p'$ be pdfs on $\mathbb{R}^d$ with supports $\mathcal{X}, \mathcal{X}'$. For thresholds $0 < \tau_\ell < \tau_u < \infty$, define 
    \begin{align*}
        \mathcal{X}_\tau := \{x : \tau_\ell \leq p(x) \leq \tau_u\}, \quad \mathcal{X}'_\tau := \{x : \tau_\ell \leq p'(x) \leq \tau_u\}, \quad \mathcal{Y}_\tau := \mathcal{X}_\tau \cup \mathcal{X}'_\tau,
    \end{align*}
    and let $u_\tau$ denote the uniform density on $\mathcal{Y}_\tau$. Write $p_\tau$ and $p'_\tau$ for the densities of $p$ and $p'$ conditional on $\mathcal{X}_\tau$ and $\mathcal{X}'_\tau$, respectively. Then $p \sim p'$ iff, for every pair $(\tau_\ell, \tau_u)$ and every $f$-divergence $D_f$:
    \begin{align*}
        D_f(p_\tau \| u_\tau) = D_f(p'_\tau \| u_\tau).
    \end{align*}
\end{prop}

The universal quantification over thresholds and $f$-divergences makes Prop. \ref{prop:diverge} especially potent.
Next, we define a partial order on distributions based on the following notion of dispersion.
\begin{figure}
    \centering
    \includegraphics[width=0.9\linewidth]{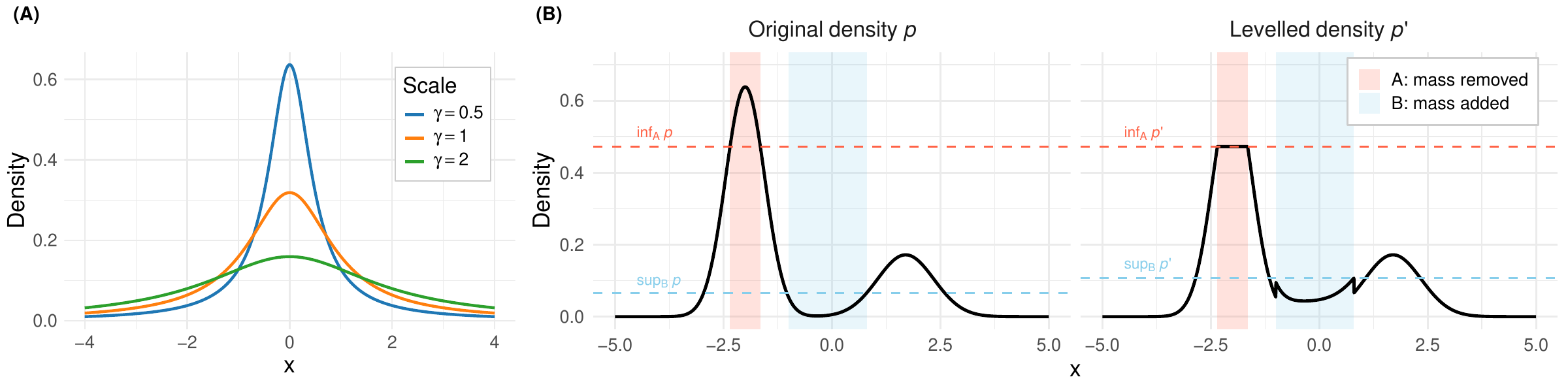}
    \caption{\textbf{(A)} Three Cauchy densities with increasing scale parameters. Levelling correctly orders the distributions by their relative uncertainty, whereas variance is undefined for all three. \textbf{(B)} Illustrative example of levelling on a bimodal distribution, where mass is redistributed from a peak to a valley.}
    \label{fig:levelling}
\end{figure}

\begin{defn}[Levelling]\label{def:levelling}
    Let $p, p'$ be pdfs on $\mathbb{R}^d$ with supports $\mathcal{X}, \mathcal{X}'$.
    We say that $p'$ is a \textit{levelling} of $p$, written $p \preceq p'$, if either of the following conditions holds.
    \begin{itemize}[noitemsep]
        \item[$\mathrm{L}1$:] There are regions $A \subset \mathcal X$ and $B \subseteq \mathcal X' \setminus A$ such that: 
        \begin{itemize}[noitemsep]
            \item[(a)] $P(X \in A) - P'(X \in A) = P'(X \in B) - P(X \in B) \geq 0$;
            \item[(b)] for all $x \not\in A \cup B, ~p(x) = p'(x)$;
            \item[(c)] for all $x \in A, p(x) \geq p'(x)$; and
            \item[(d)] $\inf_A p \geq \sup_B p$ and $\inf_A p' \geq \sup_B p'$.
        \end{itemize}
        \item[$\mathrm{L}2$:] There exists some density $p'' \sim p$ such that $p'' \preceq p'$ via $\mathrm{L}1$.
    \end{itemize}
    If $p \preceq p'$ and $p' \not \preceq p$, then we say that $p'$ is a \textit{strict levelling} of $p$, and write $p \prec p'$.
\end{defn}
Intuitively, a levelling of $p$ is a sort of Robin Hood transfer that ``flattens'' a pdf by redistributing 
%from high density to the mass-rich to the mass-poor, 
from high- to low-density regions, 
possibly on an expanded support.
Items (a) and (b) of $\mathrm{L}1$ are conservation-of-mass principles, guaranteeing that any reduction of mass in $A$ is exactly offset by an increase of mass in $B$, with no further changes beyond these subsets (see \Cref{fig:levelling}B).
$\mathrm{L}1$(c) ensures that we do not redistribute mass within $A$ in such a way that density increases in a subregion despite the transfer to $B$.
$\mathrm{L}1$(d) is a level-set condition, which states that we can only redistribute to the point of equality between subsets---possibly less, but no more. 
This prevents us from replacing one peak with another.
$\mathrm{L}2$ allows us to compare distributions that we otherwise could not, e.g. pdfs with opposite skew.
% To build on the economic analogy, levelling is a Robin Hood style redistribution that takes from the mass-rich and gives to the mass-poor.
We remark that $\preceq$ is reflexive and transitive, but antisymmetric only up to spread-equivalence. Thus $\preceq$ is a preorder on pdfs that induces a partial order on equivalence classes. In addition, levelling is closed under spread-equivalence: if $p \preceq p'$, $p \sim p_0$, and $p' \sim p'_0$, then $p_0 \preceq p'_0$.

% If $p \preceq p'$ and $p' \preceq p$, we say that the two distributions are \textit{levelling-equivalent} and write $p \sim p'$. This can occur, for example, in the case of shape-preserving shifts, such as when transforming $U(0,1)$ to $U(1, 2)$. 

With spread and levelling defined, we adapt the following definitions from previous works \citep{wimmer2023, sale_second-order-dist_2023, bulte_axiomatic_2025}. 

\begin{defn}[Spreads and shifts]\label{def:spread}
    Let $\theta \sim Q, \theta' \sim Q'$ be random vectors with $Q, Q' \in \mathcal Q(\Theta)$. We say that $Q'$ is a:
    \begin{itemize}[noitemsep]
        \item \textit{mean-preserving spread} of $Q$ iff $q \preceq q'$ and $\mathbb E_Q[\theta] = \mathbb E_{Q'}[\theta']$.
        \item \textit{spread-preserving location shift} of $Q$ iff $q \sim q'$ and $\mathbb E_Q[\theta] \neq \mathbb E_{Q'}[\theta']$.
    \end{itemize}
\end{defn}

We now introduce the modified axioms.
\begin{axiombox}
\begin{itemize}[noitemsep]
    \item[$\mathrm{A}0$:] $TU, AU$ and $EU$ should be non-negative.
    % \item[A1:] $EU(q)=0$ iff $Q = \sum_{i < \infty} w_i~\delta_{\theta_i}\text{, where} \sum_{i<\infty} w_i = 1$, $\forall i:w_i \geq 0$.
    \item[$\mathrm{A}1$:] For all $Q \in \mathcal Q(\Theta), EU(q) \geq EU(\delta_\theta)=0$. 
    \item[$\mathrm{A}2$:] If $q_\ell \preceq q \preceq q_u$, then $EU(q_\ell) \leq EU(q) \leq EU(q_u)$.
    \item[$\mathrm{A}3$:] If $p_{\ell} \preceq p \preceq p_{u}$, then $AU(p_{\ell}) \leq AU(p) \leq AU(p_{u})$.
    \item[$\mathrm{A}4$:] If $Q'$ is a mean-preserving spread of $Q$, then $EU(q) \leq EU(q')$ (weak) or \mbox{$EU(q) < EU(q')$} (strict); the same holds for $TU$.
    \item[$\mathrm{A}5$:] If $Q'$ is a spread-preserving location shift of $Q$, then $EU(q) = EU(q')$.
\end{itemize}
\end{axiombox}
Axiom $\mathrm{A}0$ is a basic requirement intended to preserve the interpretability of UMs.

$\mathrm{A}1$ formalizes the intuition that epistemic uncertainty is minimized by a Dirac mass on some $\theta \in \Theta$. 
As noted above, the second-order distribution for a consistent learning procedure approaches $\delta_{\theta^*}$ as sample size grows.

$\mathrm{A}2$ and $\mathrm{A}3$ share a similar structure and should be taken in tandem. They say that partial orderings on second- and first-order distributions imply corresponding inequalities of epistemic and aleatoric uncertainty, respectively. 
These axioms are justified by Defn. \ref{def:levelling}---the flatter and more spread out a distribution is, the greater its uncertainty. 

$\mathrm{A}4$ and $\mathrm{A}5$ are intended to capture the intuition that transformations that only impact a distribution's spread should increase uncertainty, while those that only impact location should not. 

\begin{thm}[Soundness]
    QUEST measures of $AU$ and $EU$ satisfy axioms $\mathrm{A}0$-$\mathrm{A}5$ for all $\alpha \in (0,1)$, with weak inequality in the case of $\mathrm{A}4$.
    Global $EU$ satisfies $\mathrm{A}4$ (strict) iff $q \prec q'$.
    QUEST $TU$ violates $\mathrm{A}4$ (weak and strict).
\end{thm}
The result is noteworthy because alternative measures of $AU$ and $EU$ for regression generally fail to satisfy some of these axioms. For example, differential entropy violates $\mathrm{A}0$, while variance violates $\mathrm{A}2$. Both measures violate the strict version of $\mathrm{A}4$. Other results depend on limiting the analysis to an exponential family of distributions; for full results, see \citep{bulte_axiomatic_2025}.

Some mean-preserving spreads can reduce $TU$ in our framework. Specifically, this occurs when the original model is overconfident, leading to improved calibration under a levelling of $q$.
This is a feature, not a bug---$TU$ should reward belief revisions that bring an agent's predictions closer to the truth, even at the cost of increasing $EU$. Otherwise, we must either reward miscalibration or abandon our commitment to a subjective $EU$ measure. 
We choose to bite neither bullet, and instead reject $\mathrm{A}4$ as unsuitable for $TU$. 
For further discussion on the uncertainty axioms, see \Cref{app:axioms}. 

\section{Experiments}
\label{sec:experiments}

Code for reproducing all results is available online.\footnote{\url{https://github.com/sjgoring/modal-uq-public}.}
We evaluate UMs in settings designed to isolate behaviour under
unimodal and multimodal distribution structure. 
We focus on \textit{selective prediction}, in which predictions are made on only the top proportion $r$ of most certain samples as ranked by some UM. A good measure should lead to monotonically increasing loss on a test set as $r$ increases.
%
%\paragraph{Selective Prediction.}
\paragraph{Design}
Given a training dataset $\mathcal D := \{x_i, y_i\}_{i=1}^n$ sampled from a synthetic DGP, our goal is to estimate the MAP $y^* := \argmax_{y} ~p_{\theta^*}(y \mid x)$\footnote{We predict $Y \mid x$, hence we change from speaking of distributions $X$,$\Theta$ to conditional distributions $Y \mid X$, $\Theta \mid X$.} for each $x \in \mathcal X$ in a test set. 
Methods are provided access to the same deep ensemble \citep{lakshminarayanan_simple_2017}, designed with mixture of experts (MoE) layers \citep{jacobs1991adaptive}, to produce a posterior predictive $\hat p(y \mid x) := \int_{\theta} ~p_\theta(y \mid x) ~q(\theta \mid x, \mathcal D) ~d\theta$.
We compute MAP estimates $\hat y$ via grid search over $\hat p$ and evaluate performance using the log-likelihood ratio loss $L(\hat y) = \log p_{\theta^*}(y^* \mid x) - \log p_{\theta^*}(\hat y \mid x)$.
We report loss--retention curves in \Cref{fig:selective}, where lower area under the curve denotes a more accurate UM.
% Given model predictions $\hat{y}(x)$ and an uncertainty score $u(x)$ for each test input, we rank samples by uncertainty and retain the least-uncertain proportion $r \in [0,1]$.
% We report loss--retention curves, utilising log-likelihood ratio loss.
% We say an uncertainty measure performs better it yields lower average loss at higher retention.

\label{sec:data}
We use three DGPs of increasing structural complexity: unimodal Gaussian, unimodal skewed and multimodal mixture.
\label{sec:models}
%We use a deep ensemble \citep{lakshminarayanan_simple_2017} with mixture of experts (MoE) layers \citep{jacobs1991adaptive, shazeer2017outrageously} [TODO - parameters].
Our ensemble has $M=10$ networks with $K=3$ experts per network, where each expert is a conditional unimodal Gaussian estimator.
%\paragraph{Uncertainty Measures and Metrics}
Evaluated UMs include $TU$ measures of variance, differential entropy, local QUEST ($\alpha$-volume) and global QUEST (integrated volume). 
%Each measure is decomposed into total (TU), aleatoric (AU), and epistemic (EU) components.
We assess performance from the perspective of a density oracle with access to $p_{\theta^*}$; results using maximum likelihood approximations are reported in \Cref{appx:selective-no-oracle}. 
% (For a discussion, see \Cref{sect:discussion}.) 
% That is, for the purpose of UM calculation, instead of estimating the true distribution (see \citep{schweighofer2025}) we evaluate the measures assuming access to a true density.
% This is utilised, for example in the TV-distance term in both local and global QUEST TU.
%for example the KL-divergence term in differential entropy EU.
% We acknowledge the limitation of this choice in \cref{sect:discussion}.
%Evaluation focuses on mode-aware performance:
%modal absolute error (MoAE), modal squared error (MoSE), selective prediction AURC, and active learning learning curves versus labelled budget.
We repeat each experimental configuration 20 times, and plot average values with their associated standard errors. Further detail can be found in \Cref{sec:appendix_experiments} where we also provide details of a supplementary active learning experiment.

\paragraph{Results}

We hypothesised that QUEST would provide more reliable uncertainty rankings by measuring density concentration around dominant high-density regions, thereby improving performance in selective prediction.
The inability of pointwise risk estimators to correctly quantify uncertainty is expected to be more pronounced in skewed and multimodal settings.
%, and we would expect QUEST to outperform competitors in these settings.
% Our results are outlined in Fig. \ref{fig:selective}. 
We find that global QUEST demonstrates superior performance above not only the random baseline, but variance, entropy and local QUEST in all trials. 
(Poor performance across UMs in the Gaussian setting is likely attributable to the unnecessary complexity of a deep MoE ensemble for this case.)
Local QUEST fares no worse than competitors in all settings and in the bimodal outperforms variance. 
% We observe an unexpected inverse relation between data complexity and QUEST performance, which suggests QUEST is a measure with utility in settings beyond those which first motivated its creation. 
Results of supplementary active learning experiment shows QUEST performance in line with that of variance and entropy.

\begin{figure}
    \centering
    \includegraphics[width=0.9\linewidth]{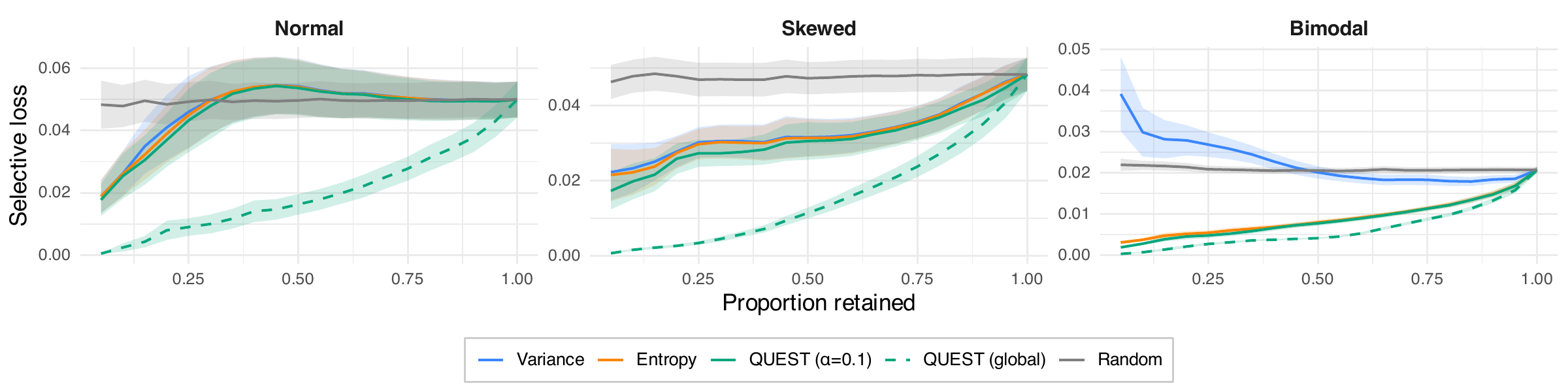}
    \caption{Selective learning curves for deep MoE ensembles on three different DGPs. Shading represents standard errors across 20 trials.}
    \label{fig:selective}
\end{figure}

\section{Discussion}\label{sect:discussion}

%\paragraph{Contributions} 
At a conceptual level, QUEST makes two independent contributions. First, we propose a volume-centric UQ framework that quantifies uncertainty via concentration of measure at one or several thresholds $\alpha$. 
Second, we adopt a purely subjective notion of $EU$ that is not truth-relative. Though these components could in principle be decoupled, we present them as a unified whole. The former follows from basic geometric intuitions; the latter from careful reflection on the proposed UM axioms. For a more sustained discussion on those axioms and the trade-offs they imply, see \Cref{app:axioms}.

%\paragraph{Limitations }
%\label{sec:limitations}
Several limitations merit acknowledgment. First, our theoretical results assume densities are absolutely continuous w.r.t. Lebesgue measure, ruling out distributions with atoms or support on lower-dimensional submanifolds. Extending the framework to handle such cases---for example, via reference measures other than Lebesgue---is a natural direction for future work. Second, $\alpha$-volume and integrated volume can be costly to compute for high-dimensional or non-analytic densities, where HDR boundaries must be approximated numerically. Efficient estimation procedures, perhaps leveraging recent advances in normalising flows \citep{Papamakarios2021} or score-based density estimation \citep{song2020}, would substantially broaden the framework's applicability. 

% Our conception of uncertainty brings some limitations, namely because TU and EU measure beliefs rather than intrinsic quantities, they are inherently subjective.  
% Consequently, it is possible---in cases of overconfidence---for estimated total uncertainty to fall below the ground-truth aleatoric uncertainty.  
% Such behaviour is consistent with the subjective Bayesian stance adopted here, and should diminish asymptotically under a consistent estimator. That these limitations exist for our framework is in our view not fatal.
% Rather the approach is aligned with the emerging consensus that no single uncertainty measure adequately serves all inferential goals and settings \citep{mucsanyi_benchmarking_2024, hofman_uncertainty_2025}.

Our experimental results rely on access to an oracle density, which is unavailable outside of experimental settings. The question of how to evaluate UMs without access to such privileged information is somewhat fraught; see \citep{schweighofer2025}. Common options include Bayesian model averaging, maximum likelihood estimation, or alternative hybrids. Each method has its strengths and weaknesses, but we judge such debates to be orthogonal to our main focus, which is the introduction of a new UQ framework. Systematically examining the impact of such choices on measure performance is a promising direction for future work.

HDRs come with optimal false coverage guarantees via the Ghosh-Pratt lemma \citep{ghosh_relation_1961, pratt_shorter_1963}. A related result is that HDRs are intrinsically linked to acceptance regions for uniformly most powerful tests \citep{Lehmann2005}. Investigating the link between hypothesis testing and uncertainty quantification may provide more machinery for UQ.
Similarly, the connection to Lorenz dominance suggests that other tools from the inequality measurement literature---Atkinson indices, generalized entropies---may yield further insights for UQ.
% Some of the limitations outlined in our work reflect broader limitations arising from dominant assumptions in the literature, for example the assumption of broad access to predictive densities. Examining uncertainty estimates in limited density access settings would prove an insightful contribution, broadening the applicability of UQ.
Whilst there has been some work on volumetric measures of uncertainty (e.g \citep{sale_is_2023}), we hope that QUEST will spur further investigation into the geometric foundations of UQ.

\section{Conclusion}

We have introduced QUEST, a framework for uncertainty quantification grounded in the Lebesgue measure of highest density regions. By centering UQ on density concentration rather than pointwise predictive risk, QUEST avoids the pathologies of PSR-based measures. We established connections between our measures and classical statistics from information theory and economics, showing that QUEST sits naturally within a broader analytical tradition. Our axiomatic analysis demonstrates that QUEST measures of $AU$ and $EU$ are sound w.r.t. a standard set of UQ axioms, while our experiments show favourable performance against state-of-the-art alternatives in selective prediction. 
Future work will explore extensions to statistical inference, as well as generalizations that do not require second-order distributions.
% Whilst there are some limitation to the conception, theoretical and experimental arguments for this work we believe QUEST is a promising contribution to the suite of measures available and believe it poses interesting directions for future work.

% For NeurIPS camera ready this must be commented out, and checklist Ethics and AI safety sections used instead.
\section*{Ethics and AI Safety statement}
This work develops methods for uncertainty quantification (UQ), which can improve the reliability of AI systems and support safer decision-making by reducing overconfident or erroneous predictions. Such improvements are particularly valuable in domains where safety criteria are well-established (e.g. healthcare or engineering systems).

However, the concept of “safety” is not universally agreed upon, particularly in socio-technical domains such as social media, surveillance, or automated decision-making. In such settings, the same UQ techniques may be used to support applications that some stakeholders consider harmful—for example, by enabling more effective targeting, filtering, or behavioural influence.

As methodological researchers, we have limited control over downstream applications of these techniques. Nevertheless, we recognise that UQ methods are dual-use: while they can reduce harmful outcomes, they may also increase the effectiveness or credibility of systems deployed in ethically contested contexts.

This work meets the ethical requirements of King’s College London and is primarily motivated by applications in domains with widely accepted safety objectives. We emphasise, however, the importance of careful consideration when applying these methods, including transparency about uncertainty, appropriate domain-specific evaluation, and awareness of potential societal impacts.

We encourage practitioners and users of this work to critically assess both the benefits and risks of deploying UQ-enabled systems in their specific context.
 
\begin{ack}
SG is supported by UK Research and Innovation [grant number EP/S023356/1], in the UKRI Centre for Doctoral Training in Safe and Trusted Artificial Intelligence.\footnote{\url{www.safeandtrustedai.org}.}
DSW was supported by EPSRC grant number UKRI918.
We gratefully acknowledge the use of the CREATE HPC system for conducting our experiments \citep{KCL_CREATE_2022}.
\end{ack}

% \newpage
\small{
\bibliography{biblio}
}
%%%%%%%%%%%%%%%%%%%%%%%%%%%%%%%%%%%%%%%%%%%%%%%%%%%%%%%%%%%%

\newpage
\appendix

\section{Proofs}\label{sec:proofs}

In this section, we include proofs for all propositions and theorems.

\textbf{Proposition}. For any pdf $p$ on $\mathcal X \subseteq \mathbb{R}^d$  and any $\alpha \in (0, 1)$ such that $V_\alpha(p) < \infty$:
    \begin{align*}
        \log V_\alpha(p) = h(p_\alpha, u_\alpha),
    \end{align*}
    where $p_\alpha := p(\cdot \mid X \in C_\alpha)$ and $u_{\alpha}$ is the uniform density on  $C_\alpha$.

\begin{proof}
On $C_\alpha$ we have $u_\alpha(x) = 1/V_\alpha(p)$, hence $\log u_\alpha(x) = -\log V_\alpha(p)$ for every $x \in C_\alpha$. The conditional density $p_\alpha$ is supported on $C_\alpha$ and integrates to one, so
\begin{align*}
    h(p_\alpha, u_\alpha)
    &= -\int_{C_\alpha} p_\alpha(x) \log u_\alpha(x)\, dx\\
    &= \log V_\alpha(p) \cdot \int_{C_\alpha} p_\alpha(x)\, dx\\
    &= \log V_\alpha(p),
\end{align*}
which is the claimed identity.
\end{proof}

\textbf{Remark} (Decomposition into entropy and KL).
Using the standard relation $h(p, r) = h(p) + D_{\mathrm{KL}}(p \,\|\, r)$ for any density $q$ and reference $r$, the proposition rewrites as
\begin{align*}
    \log V_\alpha(p) = h(p_\alpha) + D_{\mathrm{KL}}(p_\alpha \,\|\, u_\alpha),
\end{align*}
exposing log HDR volume as the conditional entropy on the HDR plus a divergence term measuring deviation of $p_\alpha$ from uniform on $C_\alpha$.

\textbf{Proposition}. Let $p$ be a density on $[0, 1]^d$. Then we have the following two identities.
    \begin{itemize}
        \item[(a)] For every $\alpha \in (0, 1)$, the Lorenz $p$-curve satisfies $L_p\big(1 - V_\alpha(p)\big) = \alpha$.
        Equivalently, the curves $\alpha \mapsto V_\alpha(p)$ and $z \mapsto L_p(z)$ are related by the coordinate flip $(\alpha, V_\alpha) \mapsto (1 - V_\alpha, \alpha)$.
        \item[(b)] The Gini $p$-coefficient satisfies $G_p = 1 - 2 \cdot IV(p)$.
    \end{itemize}

\begin{proof}

The strategy is to identify the set $\{x : p(x) \leq q_p(1 - V_\alpha(p))\}$ as the complement of the HDR $C_\alpha(p)$, then evaluate the Lorenz curve at this point directly.

Let $z^* := 1 - V_\alpha(p)$. Since $V_\alpha(p) \in (0, 1)$ for any $\alpha \in (0, 1)$ on the unit cube, we have $z^* \in (0, 1)$.

\textbf{Step 1}: \textit{The quantile $q_p(z^*)$ equals the HDR threshold}. Define the \textit{survival function of the density} $\mu_p : [0, \infty) \to [0, 1]$ by
\begin{align*}
    \mu_p(t) := \lambda\big(\{x : p(x) > t\}\big),
\end{align*}
where $\lambda$ denotes Lebesgue measure. By definition,
\begin{align*}
    q_p(z^*) = \inf \big\{t \geq 0 : \lambda(\{x : p(x) \leq t\}) \geq z^*\big\}.
\end{align*}
The set $\{x : p(x) \leq t\}$ has Lebesgue measure $1 - \mu_p(t)$, since $\lambda(\mathcal{X}) = 1$. So
\begin{align*}
    q_p(z^*) = \inf\{t \geq 0 : 1 - \mu_p(t) \geq z^*\} = \inf\{t \geq 0 : \mu_p(t) \leq 1 - z^* = V_\alpha(p)\}.
\end{align*}

Recall that the HDR threshold at level $\alpha$ is defined as $t_\alpha(p) = \sup\big\{t \geq 0 : \int_{\{p \geq t\big\}} p\, d\lambda \geq 1 - \alpha\}$, and the HDR is $C_\alpha = \big\{x : p(x) \geq t_\alpha(p)\big\}$ with $V_\alpha(p) = \lambda(C_\alpha)$. Under the level-set condition $\lambda(\{p = t\}) = 0$ for all $t > 0$,\footnote{This condition ensures that HDRs are uniquely defined and that $V_\alpha$ is absolutely continuous in $\alpha$. This assumption is convenient but not necessary. For densities with plateaus, the proof extends by taking the minimal HDR at each point and replacing the change of variables argument in Step 3 below with a Lebesgue-Stieltjes formulation. The conclusion is unchanged.} we have $\lambda(C_\alpha) = \lambda(\{p > t_\alpha(p)\}) = \mu_p(t_\alpha(p))$, hence
\begin{align*}
    \mu_p(t_\alpha(p)) = V_\alpha(p).
\end{align*}
Combined with the fact that $\mu_p$ is strictly decreasing on the range of HDR thresholds under our stated regularity condition, this gives $\inf\{t \geq 0 : \mu_p(t) \leq V_\alpha(p)\} = t_\alpha(p)$, and therefore $q_p(z^*) = t_\alpha(p)$.

\textbf{Step 2}: \textit{The $p$-Lorenz curve evaluation}. By definition,
\begin{align*}
    L_p(z^*) = \int_{\{x : p(x) \leq q_p(z^*)\}} p(x)\, dx = \int_{\{x : p(x) \leq t_\alpha(p)\}} p(x)\, dx.
\end{align*}
Under the level-set condition, the sets $\{p \leq t_\alpha\}$ and $\{p < t_\alpha\}$ differ by a Lebesgue null set, and the latter is the complement of the HDR $C_\alpha = \{p \geq t_\alpha\}$. So
\begin{align*}
    L_p(z^*) = \int_{C_\alpha^c} p(x)\, dx = 1 - \int_{C_\alpha} p(x)\, dx = 1 - (1 - \alpha) = \alpha.
\end{align*}
Substituting $z^* = 1 - V_\alpha(p)$ gives the claimed identity of item (a).

\textbf{Step 3}: \textit{Evaluating the $p$-Gini coefficient.} By (a), the Lorenz curve $L_p$ and the coverage curve are related by the coordinate flip $L_p(1 - V_\alpha(p)) = \alpha$. We compute the Gini coefficient by a change of variables.

By definition,
\begin{align*}
    G_p = 1 - 2\int_0^1 L_p(z)\, dz.
\end{align*}
Substitute $z = 1 - V_\alpha(p)$, so $dz = -V'_\alpha(p)\, d\alpha$, where $V'_\alpha$ denotes the derivative of $V_\alpha$ with respect to $\alpha$. By (a), $L_p(z) = \alpha$ at this substitution. The integration limits transform: as $\alpha$ ranges over $(0, 1)$, $z = 1 - V_\alpha(p)$ ranges from $1 - V_0(p) = 0$ at $\alpha = 0$ (since $V_0(p) = 1$ on the unit cube) to $1 - V_1(p) = 1$ at $\alpha = 1$ (since $V_1(p) = 0$). Note that $V_\alpha$ is non-increasing in $\alpha$, so $V'_\alpha \leq 0$ and $-V'_\alpha \geq 0$.

The integral becomes
\begin{align*}
    \int_0^1 L_p(z)\, dz = \int_0^1 \alpha \cdot (-V'_\alpha(p))\, d\alpha.
\end{align*}
Integrate by parts with $f(\alpha) = \alpha$ and $g'(\alpha) = -V'_\alpha(p)$, hence $f'(\alpha) = 1$ and $g(\alpha) = -V_\alpha(p)$:
\begin{align*}
    \int_0^1 \alpha \cdot (-V'_\alpha(p))\, d\alpha = \big[\alpha \cdot (-V_\alpha(p))\big]_0^1 + \int_0^1 V_\alpha(p)\, d\alpha = (-V_1(p)) - 0 + IV(p) = IV(p),
\end{align*}
using $V_1(p) = 0$.

Therefore,
\begin{align*}
    G_p = 1 - 2 \cdot IV(p),
\end{align*}
which concludes the proof.

\end{proof}

\textbf{Proposition.}
Let $p, p'$ be pdfs on $\mathbb{R}^d$ with supports $\mathcal{X}, \mathcal{X}'$. For thresholds $0 < \tau_\ell < \tau_u < \infty$, define
\begin{align*}
    \mathcal{X}_\tau := \{x : \tau_\ell < p(x) \leq \tau_u\}, \quad \mathcal{X}'_\tau := \{x : \tau_\ell < p'(x) \leq \tau_u\}, \quad \mathcal{Y}_\tau := \mathcal{X}_\tau \cup \mathcal{X}'_\tau,
\end{align*}
and let $u_\tau$ denote the uniform density on $\mathcal{Y}_\tau$. Write $p_\tau, p'_\tau$ for the conditional densities of $p, p'$ on their respective level-set windows. Then $p \sim p'$ if and only if, for every threshold pair $(\tau_\ell, \tau_u)$ and every $f$-divergence $D_f$:
\begin{align*}
    D_f(p_\tau ~||~ u_\tau) = D_f(p'_\tau ~||~ u_\tau).
\end{align*}

\begin{proof}
For thresholds $(\tau_\ell, \tau_u)$ where $V_\tau := \lambda(\mathcal{X}_\tau)$ and $V'_\tau := \lambda(\mathcal{X}'_\tau)$ are both positive, write $W_\tau := \lambda(\mathcal{Y}_\tau)$, $m_\tau := \int_{\mathcal{X}_\tau} p\, d\lambda$, $m'_\tau := \int_{\mathcal{X}'_\tau} p'\, d\lambda$. Let $\mu_p(t) := \lambda(\{p > t\})$ denote the level-set volume function, so spread-equivalence asserts $\mu_p \equiv \mu_{p'}$.

\textbf{Step 1: Divergence decomposition.}
Since $p_\tau$ vanishes outside $\mathcal{X}_\tau$ but $u_\tau = W_\tau^{-1}$ is positive throughout $\mathcal{Y}_\tau$:
\begin{align*}
    D_f(p_\tau \| u_\tau)
    = \frac{1}{W_\tau}\int_{\mathcal{X}_\tau} f\!\left(\frac{W_\tau\, p(x)}{m_\tau}\right) dx + \frac{f(0)\,(W_\tau - V_\tau)}{W_\tau}.
\end{align*}
Multiplying through by $W_\tau$, the divergence equality is equivalent to
\begin{align}\label{eq:dagger}
    A_f(p, \tau) - f(0) V_\tau = A_f(p', \tau) - f(0) V'_\tau,
\end{align}
where
\begin{align*}
    A_f(p, \tau) := \int_{\mathcal{X}_\tau} f\!\left(\frac{W_\tau\, p(x)}{m_\tau}\right) dx.
\end{align*}

\textbf{Step 2: Forward direction.}
Suppose $\mu_p \equiv \mu_{p'}$. Then for any window:
\begin{align*}
    V_\tau = \mu_p(\tau_\ell) - \mu_p(\tau_u) = V'_\tau,
\end{align*}
so the $f(0) V_\tau$ terms in Eq. \ref{eq:dagger} match. By layer-cake, using $\lambda(\{p > s\} \cap \mathcal{X}_\tau) = V_\tau$ for $s < \tau_\ell$, $= \mu_p(s) - \mu_p(\tau_u)$ for $s \in [\tau_\ell, \tau_u]$, and $= 0$ for $s \geq \tau_u$:
\begin{align*}
    m_\tau = \tau_\ell V_\tau + \int_{\tau_\ell}^{\tau_u} \mu_p(s)\, ds - (\tau_u - \tau_\ell)\mu_p(\tau_u),
\end{align*}
which depends only on $\mu_p$ on $[\tau_\ell, \tau_u]$ and at the endpoints. Spread-equivalence therefore gives $m_\tau = m'_\tau$.

The conditional CDF of $p(X)$ given $X \in \mathcal{X}_\tau$ (with $X$ uniform on $\mathcal{X}_\tau$) at level $s \in [\tau_\ell, \tau_u]$ equals 
\begin{align*}
    \frac{\mu_p(s) - \mu_p(\tau_u)}{V_\tau},
\end{align*}
which is determined by $\mu_p$ alone. Spread-equivalence gives equality of these CDFs across $p$ and $p'$. Combined with $W_\tau / m_\tau = W_\tau / m'_\tau$ being a deterministic monotone rescaling, the two random variables $W_\tau p(X)/m_\tau \,|\, X \in \mathcal{X}_\tau$ and $W_\tau p'(X')/m'_\tau \,|\, X' \in \mathcal{X}'_\tau$ are equal in distribution. Hence
\begin{align*}
    A_f(p, \tau) = V_\tau \, \mathbb{E}\!\left[f\!\left(\tfrac{W_\tau p(X)}{m_\tau}\right) \,\bigg|\, X \in \mathcal{X}_\tau\right] = V'_\tau \, \mathbb{E}\!\left[f\!\left(\tfrac{W_\tau p'(X')}{m'_\tau}\right) \,\bigg|\, X' \in \mathcal{X}'_\tau\right] = A_f(p', \tau),
\end{align*}
which together with $V_\tau = V'_\tau$ gives Eq. \ref{eq:dagger}.

\textbf{Step 3: Reverse direction.}
Suppose that $D_f(p_\tau ~||~ u_\tau) = D_f(p'_\tau ~||~ u_\tau)$ for every $f$-divergence and threshold pair $\tau = (\tau_\ell, \tau_u)$.
Define 
\begin{align*}
    Y := \frac{W_\tau p(X)}{m_\tau}, \quad Y' := \frac{W_\tau p(X)}{m'_\tau}.
\end{align*}
Since $\tau_\ell > 0, \tau_u < \infty$, and $m_\tau, m'_\tau \in (0, \infty)$, both $Y$ and $Y'$ are supported on compact subsets of $(0, \infty)$. Let $\nu, \nu'$ denote their respective laws.

Rearranging the divergence decomposition of Step 1 gives
\begin{align*}
    W_\tau D_f(p_\tau ~||~ u_\tau) = V_\tau \mathbb E\big[f(Y)\big] + f(0)(W_\tau - V_\tau),
\end{align*}
and analogously for $p'_\tau$. Since the divergences are equal by assumption, we have
\begin{align*}
	V_\tau \mathbb E\big[f(Y)\big] + f(0)(W_\tau-V_\tau) = V'_\tau \mathbb E\big[f(Y')\big] + f(0)(W_\tau-V'_\tau).
\end{align*}
Rearranging gives
\begin{align*}
	V_\tau {\mathbb E\big[f(Y)\big]-f(0)} = V'_\tau {\mathbb E\big[f(Y')\big]-f(0)}.
\end{align*}

Consider the family of power divergences $f_k(t)=t^k-1$ for integers $k\ge 2$. These functions are convex and satisfy $f_k(1)=0$. Since $f_k(0)=-1$, the preceding equation becomes
\begin{align*}
	V_\tau \mathbb E\big[Y^k\big] = V'_\tau \mathbb E\big[Y'^k\big], \qquad k\ge 2.
\end{align*}
This is the key moment identity. Let $K\subset(0,\infty)$ be a compact interval containing the supports of both $\nu$ and $\nu'$. Define a finite signed measure $\sigma$ on $K$ by
\begin{align*}
	d\sigma(y) := y^2 \big(V_\tau,d\nu(y)-V'_\tau,d\nu'(y) \big).
\end{align*}
For every integer $j\ge 0$, the moment identity with $k=j+2$ gives
\begin{align*}
	\int_K y^j,d\sigma(y) = V_\tau\mathbb E[Y^{j+2}] - V'_\tau\mathbb E[Y'^{j+2}] = 0.
\end{align*}
Thus every polynomial has zero integral with respect to $\sigma$. Since polynomials are dense in $C(K)$ by the Weierstrass approximation theorem, it follows that
\begin{align*}
	\int_K g(y),d\sigma(y)=0
\end{align*}
for every continuous function $g\in C(K)$. Hence $\sigma=0$ as a finite signed measure on $K$.

Because $K\subset(0,\infty)$, the function $y\mapsto y^{-2}$ is bounded and continuous on $K$. 
Multiplying by $y^{-2}$, we obtain 
\begin{align*}
    V_\tau\,d\nu = V'_\tau\,d\nu'.
\end{align*}
Taking total masses on both sides yields 
\begin{align*}
    V_\tau=V'_\tau,
\end{align*}
since $\nu$ and $\nu'$ are probability measures. 
Therefore, for every admissible threshold pair $(\tau_\ell,\tau_u), \lambda(\mathcal X_\tau)=\lambda(\mathcal X'_\tau)$.  
Equivalently, we have
\begin{align*}
	\lambda\{x:\tau_\ell < p(x)\le \tau_u\}=\lambda\{x:\tau_\ell < p'(x)\le \tau_u\}.
\end{align*}
In terms of the superlevel-volume functions
\begin{align*}	
	\mu_p(t):=\lambda\{x:p(x)>t\}, \qquad \mu_{p'}(t):=\lambda\{x:p'(x)>t\}.
\end{align*}
This says that
\begin{align*}
	\mu_p(\tau_\ell)-\mu_p(\tau_u)=\mu_{p'}(\tau_\ell)-\mu_{p'}(\tau_u)
\end{align*}
for all $0<\tau_\ell<\tau_u<\infty$, up to the usual endpoint convention for level sets. 
Hence the difference 
\begin{align*}
    \Delta(t):=\mu_p(t)-\mu_{p'}(t)
\end{align*}
is constant on $(0,\infty)$. 
Since $p$ and $p'$ are probability densities, Markov’s inequality gives
\begin{align*}
	\mu_p(t)\le \frac{1}{t}, \qquad \mu_{p'}(t)\le \frac{1}{t},
\end{align*}
and therefore both $\mu_p(t)$ and $\mu_{p'}(t)$ converge to $0$ as $t\to\infty$. The constant value of $\Delta(t)$ must therefore be $0$. Thus
\begin{align*}
	\mu_p(t)=\mu_{p'}(t) \qquad \text{for all }t>0,
\end{align*}
which is precisely spread-equivalence, $p\sim p'$.
\end{proof}

\textbf{Theorem} (Soundness). QUEST measures of $AU$ and $EU$ satisfy axioms $\mathrm{A}0$-$\mathrm{A}5$ for all $\alpha \in (0,1)$, with weak inequality in the case of $\mathrm{A}4$.
Global $EU$ satisfies $\mathrm{A}4$ (strict) iff $q \prec q'$.
$TU$ violates $\mathrm{A}4$ (weak and strict).

\begin{proof}

We begin by deriving a crucial lemma, which will be repeatedly invoked below.

\textbf{Lemma} (Partial ordering). Suppose that $p \preceq p'$. Then (a) for any $\alpha \in (0, 1), V_\alpha(p) \leq V_\alpha(p')$; and (b) $IV(p) \leq IV(p')$, with equality in both cases iff $p \sim p'$.

\begin{proof}
Our strategy is to show that levelling implies an ordering on cumulative density profiles, from which the HDR-volume inequalities follow directly.

For a density $p$ on $\mathcal X \subseteq \mathbb R^d$, define the survival function of the density $\mu_p : [0,\infty) \to [0,\infty]$ by
\begin{align*}
	\mu_p(t) := \lambda\big(\{x\in \mathcal X : p(x)>t\}\big),
\end{align*}
where $\lambda$ denotes Lebesgue measure. Markov's inequality gives $\mu_p(t)\leq 1/t$ for $t>0$, so $\mu_p$ is finite on $(0,\infty)$. The decreasing rearrangement of $p$ is
\begin{align*}
	p^\downarrow(r) := \inf \{t\geq 0 : \mu_p(t)\leq r\}, \qquad r\geq 0.
\end{align*}
This function is non-increasing, right-continuous, and equimeasurable with $p$, meaning
\begin{align*}
	\lambda\big(\{x:p(x)>t\}\big) = \lambda_1\big(\{r:p^\downarrow(r)>t\}\big)
\end{align*}
for all $t\geq 0$, where $\lambda_1$ denotes one-dimensional Lebesgue measure. Define the cumulative profile
\begin{align*}
	\mathcal P(r) := \int_0^r p^\downarrow(s),ds, \qquad r\geq 0.
\end{align*}
The function $\mathcal P$ is non-decreasing, continuous, and bounded above by $1$. Define $\mathcal P'$ analogously from $p'$.

\textbf{Step 1: HDR volume in terms of the cumulative profile.}
Under the level-set assumption $\lambda(\{p=t\})=0$ for all $t>0$, the HDR threshold
\begin{align*}
	t_\alpha(p) := \sup~ \left\{t:\int_{{p\geq t}}p ~d\lambda\geq 1-\alpha\right\}
\end{align*}
satisfies
\begin{align*}
	\int_{C_\alpha(p)}p ~d\lambda = \int_0^{V_\alpha(p)}p^\downarrow(s) ~ds = \mathcal P\big(V_\alpha(p)\big) = 1-\alpha.
\end{align*}
The first equality follows from equimeasurability of $p$ and $p^\downarrow$, since the HDR corresponds to the upper level set of $p$, or equivalently to the initial segment of its decreasing rearrangement. Thus, with the usual generalized inverse convention,
\begin{align*}
	V_\alpha(p) = \mathcal P^{-1}(1-\alpha) := \inf \{r\geq 0:\mathcal P(r)\geq 1-\alpha \}.
\end{align*}
The same representation holds for $p'$:
\begin{align*}
	V_\alpha(p') = \mathcal P'^{-1}(1-\alpha).
\end{align*}
The level-set assumption is convenient but not essential. If $p$ has plateaus, we take $V_\alpha(p)$ to be the minimal HDR volume and interpret the preceding display using the generalized inverse; the same argument then goes through by right-continuity.

\textbf{Step 2: Levelling implies majorization.}
We first consider the case where $p\preceq p'$ via $\mathrm{L}1$. Let $A$ and $B$ be the donor and recipient regions in Defn.~\ref{def:levelling}. Define
\begin{align*}
	a(x)&:=p(x)-p'(x)\geq 0, \qquad x\in A\\
    b(x)&:=p'(x)-p(x)\geq 0, \qquad x\in B.
\end{align*}
By $\mathrm{L}1$(a), mass is conserved:
\begin{align*}
	\int_A a(x),d\lambda(x) = \int_B b(x),d\lambda(x).
\end{align*}
By $\mathrm{L}1$(b), $p$ and $p'$ agree outside $A\cup B$. Thus, to compare the spread of $p$ and $p'$, it suffices to compare the loss on $A$ with the gain on $B$.

We use the standard positive-part characterization of majorization \citep{marshall2011inequalities}: for non-negative integrable functions with equal total integral, $p^\downarrow$ majorizes $p'^\downarrow$ iff
\begin{align*}
	\int (p-c)_+ ~d\lambda \geq \int (p'-c)_+ ~d\lambda \qquad \forall c\geq 0.
\end{align*}
Fix $c\geq 0$. The loss in the positive-part functional caused by lowering the density on $A$ is
\begin{align*}
	L_c := \int_A \big\{(p(x)-c)_+-(p'(x)-c)_+\big\} ~d\lambda(x).
\end{align*}
Using $(u-c)_+-(v-c)_+=\int_v^u \mathbf 1\{s>c\} ~ds$ for $u\geq v$, this can be written as
\begin{align*}
	L_c = \int_A\int_{p'(x)}^{p(x)} \mathbf 1\{s>c\} ~ds ~d\lambda(x).
\end{align*}
Similarly, the gain on $B$ is
\begin{align*}
	G_c &:= \int_B \big\{(p'(x)-c)_+-(p(x)-c)_+\big\} ~d\lambda(x) \\
    &= \int_B\int_{p(x)}^{p'(x)} \mathbf 1\{s>c\} ~ds~d\lambda(x).
\end{align*}
Recall that $\mathrm{L}1$(d) prevents level crossing. If $x\in A$, then every removed layer has height $s\in[p'(x),p(x)]$. If $y\in B$, then every added layer has height $r\in[p(y),p'(y)]$. By $\mathrm{L}1$(d),
\begin{align*}
	\inf_A p \geq \sup_B p \qquad\text{and}\qquad \inf_A p' \geq \sup_B p',
\end{align*}
so every removed layer lies weakly above every added layer:
\begin{align*}
	s \geq p'(x) \geq \inf_A p' \geq \sup_B p' \geq p'(y) \geq r.
\end{align*}
Moreover, the total volume of removed layers equals the total volume of added layers, since
\begin{align*}
    \int_A\int_{p'(x)}^{p(x)}~ds ~d\lambda(x) &= \int_A a(x) ~d\lambda(x) = \int_B b(x) ~d\lambda(x) = \int_B\int_{p(x)}^{p'(x)} ~ds ~d\lambda(x).
\end{align*}
Because the removed layer-volume lies entirely at heights weakly above the added layer-volume, thresholding at any level $c$ can count no more added layer-volume than removed layer-volume. Hence
\begin{align*}
	G_c\leq L_c \qquad \forall c\geq 0.
\end{align*}
It follows that
\begin{align*}
	\int (p-c)_+ ~d\lambda - \int (p'-c)_+ ~d\lambda &= L_c-G_c \geq 0.
\end{align*}
Therefore
\begin{align*}
	\int (p-c)_+,d\lambda \geq \int (p'-c)_+,d\lambda \qquad \forall c\geq 0.
\end{align*}
By the positive-part characterization, $p^\downarrow$ majorizes $p'^\downarrow$. Equivalently,
\begin{align*}
	\mathcal P(r) = \int_0^r p^\downarrow(s) ~ds \geq \int_0^r p'^\downarrow(s) ~ds = \mathcal P'(r) \qquad \forall r\geq 0.
\end{align*}

Now suppose $p\preceq p'$ via $\mathrm{L}2$. Then there exists a density $p''$ such that $p''\sim p$ and $p''\preceq p'$ via $\mathrm{L}1$. Since spread-equivalent densities have the same decreasing rearrangement, $p''^\downarrow=p^\downarrow$ and hence $\mathcal P^{''}=\mathcal P$. Applying the preceding $\mathrm{L}1$ argument to $p''$ and $p'$ gives
\begin{align*}
	\mathcal P(r)=\mathcal P^{''}(r)\geq \mathcal P'(r) \qquad \forall r\geq 0.
\end{align*}
Thus $p\preceq p'$ implies $\mathcal P\geq \mathcal P'$ pointwise.

\textbf{Step 3: Cumulative profile ordering implies HDR-volume ordering.}
Fix $\alpha\in(0,1)$. Since $\mathcal P(r)\geq \mathcal P'(r)$ for every $r\geq 0$, the profile $\mathcal P$ reaches the level $1-\alpha$ no later than $\mathcal P'$. Using the generalized inverse representation from Step 1,
\begin{align*}
	V_\alpha(p) = \inf \{r:\mathcal P(r)\geq 1-\alpha\} \leq \inf \{r:\mathcal P'(r)\geq 1-\alpha\} = V_\alpha(p').
\end{align*}
This proves (a).

If $p\sim p'$, then $p^\downarrow=p'^\downarrow$, hence $\mathcal P=\mathcal P'$ and $V_\alpha(p)=V_\alpha(p')$ for every $\alpha$. Conversely, if $V_\alpha(p)=V_\alpha(p')$ for every $\alpha\in(0,1)$, then the generalized inverse profiles agree on $(0,1)$, so $\mathcal P=\mathcal P'$ on the relevant range. Hence $p^\downarrow=p'^\downarrow$ almost everywhere, which is equivalent to $p\sim p'$.

\textbf{Step 4: Integrated volume.}
Integrating the pointwise inequality from Step 3 gives
\begin{align*}
	IV(p) = \int_0^1 V_\alpha(p) ~d\alpha \leq \int_0^1 V_\alpha(p') ~d\alpha = IV(p').
\end{align*}
This proves (b).

If $p\sim p'$, then $V_\alpha(p)=V_\alpha(p')$ for every $\alpha$, so $IV(p)=IV(p')$. Conversely, suppose $IV(p)=IV(p')$. Since $V_\alpha(p)\leq V_\alpha(p')$ for every $\alpha$, the non-negative function $V_\alpha(p')-V_\alpha(p)$ integrates to zero over $(0,1)$. Hence $V_\alpha(p)=V_\alpha(p')$ for almost every $\alpha$. By right-continuity of $\alpha\mapsto V_\alpha(p)$ and $\alpha\mapsto V_\alpha(p')$, equality extends to every $\alpha\in(0,1)$. By the equality statement in Step 3, $p\sim p'$.
\end{proof}

\begin{table}[t]
    \centering
    \caption{QUEST measures of uncertainty. (Reproduced from \cref{tab:quest_measures} for convenience)}
    \begin{tabular}{lccc}
        \toprule
        & \textbf{Aleatoric} & \textbf{Epistemic} & \textbf{Total} \\
        \midrule
        \textbf{Local}  & $V_\alpha(p_{\theta^*})$ & $V_\alpha(q)$ & $V_\alpha(p_{\theta^*}) \cdot g_f \big( D_f(p^\downarrow_\alpha ~||~ \hat p_\alpha) \big)$ \\
        \addlinespace % Adds a small vertical gap for readability
        \textbf{Global} & $IV(p_{\theta^*})$ & $IV(q)$ & $\int_0^1 V_\alpha(p_{\theta^*}) \cdot g_f \big( D_f(p^\downarrow_\alpha ~||~ \hat p_\alpha) \big) ~d \alpha$ \\
        \bottomrule
    \end{tabular}
\end{table}

With this lemma established, we return to our axiomatic assessment. 
We reprint our table of QUEST measures above for reference.

\textit{$\mathrm{A}0$: $TU, AU$ and $EU$ should be non-negative.}
Non-negativity is immediate from the QUEST definitions, as volume and divergence are always bounded below by $0$.

\textit{$\mathrm{A}1$: For all $Q \in \mathcal Q(\Theta), EU(q) \geq EU(\delta_\theta) = 0$.}
A Dirac mass has zero volume, and so $EU(\delta_\theta)=0$ on both local and global measures. The inequality follows from the non-negativity of $EU$ in general ($\mathrm{A}0$).

\textit{$\mathrm{A}2$: If $q_\ell, q, q_u$ satisfy $q_\ell \preceq q \preceq q_u$, then $EU(q_\ell) \leq EU(q) \leq EU(q_u)$.}
This follows from the partial ordering lemma, with $EU$ measured by $V_\alpha(q)$ (local) and $IV(q)$ (global).

\textit{$\mathrm{A}3$: If $p_{\ell}, p, p_{u}$ satisfy $p_{\ell} \preceq p \preceq p_{u}$, then $AU(p_{\ell}) \leq AU(p) \leq AU(p_{u})$.}
This follows from the partial ordering lemma, with $AU$ measured by $V_\alpha(p_{\theta^*})$ (local) and $IV(p_{\theta^*})$ (global).

\textit{$\mathrm{A}4$: If $Q'$ is a mean-preserving spread of $Q$, then $EU(q) \leq EU(q')$ (weak) or $EU(q) < EU(q')$ (strict); the same holds for $TU$.}
For global $EU$, both inequalities follow immediately from the partial ordering lemma and Defn \ref{def:spread}, with strict inequality iff $q \prec q'$. For local $EU$, only the weak inequality is satisfied, since levelling may occur strictly below the $(1 - \alpha)$-HDR threshold and therefore imply no reduction in $V_\alpha(q)$.
For $TU$, we construct a simple counterexample. Suppose that $p_{\theta^*}$ is standard normal, with $q$ consisting of a Dirac mass at $\hat \mu = 1, \hat \sigma = 0$. In this case, $TU$ is infinite, as true parameter values are assigned zero density by the second-order distribution. A mean-preserving spread of $q$ that expands this distribution to include $\mu^*=0, \sigma^*=1$ in its $(1 - \alpha)$-HDR will reduce the $f$-divergence between $p^*_\alpha$ and $\hat p_\alpha$, resulting in lower $TU$. More generally, any levelling of $q$ that improves calibration at level $\alpha$ will reduce the inflation factor $g_f\big(D_f (p^*_\alpha ~||~ \hat p_\alpha) \big)$.

\textit{$\mathrm{A}5$: If $Q'$ is a spread-preserving location shift of $Q$, then $EU(q) = EU(q')$.}
This follows from the partial ordering lemma and the antisymmetry property of partial orders. 

\end{proof}

\section{Additional Experimental Details}
\label{sec:appendix_experiments}

This appendix provides extended details for the selective prediction evaluation.

Selective prediction evaluates whether uncertainty measures \emph{rank} predictions by reliability.
Given uncertainty scores $u(x)$ on a fixed test set, we sort test points by $u(x)$ and retain a subset
$S_r$ at retention proportion $r \in [0,1]$ (i.e., the least uncertain $r$ fraction). Loss at retention $r$ is
\begin{equation}
  \mathcal{L}(r) \;=\; \frac{1}{|S_r|} \sum_{(x,y)\in S_r} \ell(\hat{y}(x),y),
\end{equation}
for a suitable loss $\ell(\cdot,\cdot)$, where $\hat{y}(x)$ is the prediction of the model for $x$.
We report loss--retention curves (see below for definition of loss).% and the area under the risk--coverage curve (AURC).
Each experimental configuration is repeated 20 times. Seeds are set such that the results are reproducible.

\subsection{Data}
\label{sec:appendix_data}
We construct three datasets targeting distinct behaviours of uncertainty measures. $1500$ data points are sampled, and split randomly such that $n_{train}=1000$ and $n_{test}=500$. For each MoE model, the training data is bootstrap sampled to increase ensemble diversity. The resulting training sample is of size $n_{train}$. The equations governing the data generating processes are outlined below

\paragraph{Unimodal Gaussian}
We generate unimodal Gaussian conditionals to test that QUEST does not degrade performance in settings where variance and differential entropy
are appropriate proxies for uncertainty in a density concentration/modal perspective. 
\textbf{Expected outcome:} variance, entropy, and QUEST exhibit similar ranking quality and downstream performance.
\begin{gather*}
Y \mid x \sim \mathcal{N}(\mu(x), \sigma_G^2(x)) \\
\sigma_G(x) = \sqrt{s^2 + \delta^2(x)}, \quad s = 0.2, \quad \delta(x) = 0.5 + |x| \\
\mu(x) = \sin(\pi x), \quad X \sim \text{Uniform}[-2, 2]
\end{gather*}

\paragraph{Unimodal skewed}
We generate unimodal conditionals with marked skewness. The aim is to create regions where dispersion-based measures are dominated by tail mass rather than local density concentration.
\textbf{Expected outcome:} variance and differential entropy can mis-rank uncertainty due to tail sensitivity,
whereas QUEST improves ranking by focusing on concentration near dominant high-density regions.
\begin{gather*}
Y \mid x = \mu(x) + \sigma(x) \cdot Z, \quad Z \sim \text{SkewNormal}(\alpha(x))\\
\sigma(x) = 0.3 + 0.2|x|, \quad \alpha(x) = 5 \cdot \text{sign}(x) \sqrt{|x|}\\
\mu(x) = \sin(\pi x), \quad X \sim \text{Uniform}[-2, 2]
\end{gather*}
note $Z$ is a standardised SkewNormal distribution, with mean 0 and variance 1.
\paragraph{Multimodal mixture}
We generate input-dependent two-component Gaussian mixtures:
% \begin{equation}
%   y \mid x \sim \sum_{k=1}^{2} \pi_k(x)\,\mathcal{N}\!\big(\mu_k(x), \sigma_k^2\big),
% \end{equation}
% where $\pi_k(x)$ and $\mu_k(x)$ vary with $x$ to create regions of mode separation, overlap, and dominance.
Because the density is known, conditional modes and ground-truth uncertainty-related quantities can be computed
exactly (or to arbitrary numerical precision).
\textbf{Expected outcome:} variance and differential entropy can mis-rank uncertainty due to bi-modal nature of the data erroneously inflating these measures, whereas QUEST improves ranking by correctly focusing on concentration near local modes.
\begin{gather*}
Y \mid x \sim \frac{1}{2}\mathcal{N}(\mu(x) - \delta(x), s^2) + \frac{1}{2}\mathcal{N}(\mu(x) + \delta(x), s^2)\\
\quad s = 0.2, \quad \delta(x) = 0.5 + |x| \\
\mu(x) = \sin(\pi x), \quad X \sim \text{Uniform}[-2, 2]
\end{gather*}

\subsection{Ensemble of mixture of experts}
\label{app:moe-description}
For our model we utilise an ensemble of $M=10$ independently trained mixture of expert models. The predictive distribution is given by:
\begin{equation*}
p(y \mid x) = \frac{1}{M} \sum_{m=1}^{M} p^{(m)}(y \mid x),
\end{equation*}
where each member $m$ is a mixture of $K=3$ unimodal Gaussian experts:
\begin{equation*}
p^{(m)}(y \mid x) = \sum_{k=1}^{K} g_k^{(m)}(x)\,
\mathcal{N}\big(y \mid \mu_k^{(m)}(x), (\sigma_k^{(m)}(x))^2\big).
\end{equation*}

Here $g_k^{(m)}(x)$ are softmax gating weights:
\begin{equation*}
g_k^{(m)}(x) = \frac{\exp(a_k^{(m)}(x))}{\sum_{j=1}^{K} \exp(a_j^{(m)}(x))}.
\end{equation*}

$a_k^{(m)}(x)$ denotes the unnormalized gating score for expert $k$ in ensemble member $m$, and the corresponding mixture weight is obtained by a softmax transformation,
where the score is defined as:
\begin{equation*}
a_k^{(m)}(x) = \log \alpha_k^{(m)} + \log \mathcal{N}\big(y \mid \mu_k^{(m)}(x), (\sigma_k^{(m)}(x))^2\big).
\end{equation*}
where $\alpha_k^{(m)} \ge 0$ are mixture weights of the experts, learnt via the EM algorithm with k-means initialisation. Note
\begin{equation*}
\sum_{k=1}^{K} \alpha_k^{(m)} = 1.
\end{equation*}

Each model $m$ is trained independently via the EM agorithm, maximising log loss over the parameter space:
\begin{equation*}
\mathcal{L}^{(m)} = \sum_{n=1}^{N} \log p^{(m)}(y_n \mid x_n).
\end{equation*}

We utilise the MoE implementation from the package \textit{cgmm} (see \cref{subsec:appx-python}), and create an ensemble of these within our code.
Each \textit{cgmm} MoE is trained with hyper-parameters outlined in \cref{tab:cgmm_moe_hyperparams_full}.

\begin{table}[t]
\centering
\caption{Mixture-of-Experts hyperparameters for the cgmm-based MoE implementation. Bold entries indicate values that are differ from \textit{cgmm} defaults.}
\begin{tabular}{lll}
\hline
Parameter & Value & Explanation \\
\hline
\texttt{n\_components} & 3 & Number of experts per MoE member \\
\texttt{covariance\_type} & \texttt{full} & Covariance structure for each expert \\
\texttt{shared\_covariance} & \texttt{False} & Whether experts share a covariance matrix \\
\texttt{mean\_function} & \texttt{affine} & Expert mean-function class \\
\texttt{reg\_covar} & $\mathbf{10^{-4}}$ & Covariance regularization for numerical stability \\
\texttt{gating\_penalty} & $10^{-2}$ & L2 penalty on gating weights \\
\texttt{gating\_max\_iter} & 50 & Maximum optimization iterations for gating updates \\
\texttt{gating\_penalty\_bias} & \texttt{None} & Optional bias penalty term \\
\texttt{gating\_tol} & $10^{-6}$ & Gating convergence tolerance \\
\texttt{gating\_init\_scale} & 0.1& Initial scale for gating parameters \\
\texttt{max\_iter} & 200 & Maximum EM iterations \\
\texttt{tol} & $10^{-4}$ & EM convergence tolerance \\
\texttt{n\_init} & 1 & Number of random initializations \\
\texttt{init\_params} & \texttt{kmeans} &Initialization method \\
\hline
\end{tabular}
\label{tab:cgmm_moe_hyperparams_full}
\end{table}

\subsection{Uncertainty Measures}
Computation of uncertainty measures utilises empirical density samples, with grid based approximations when no closed form analytic solution exists (e.g. for obtaining highest density regions).
We state the measures as utilised in our computations in \cref{tbl:appx-UMs}.
Note we use the Oracle perspective in the selective prediction experiment. Entries with $p^*$ do not require this to be approximated as would usually happen in settings without access to the true density.
\begin{table}[ht]
\centering
\caption{Uncertainty measures as computed for our experiments}
\begin{tabular}{LLLL}
    \hline
     & \textbf{TU} & \textbf{AU} &\textbf{EU} \\ \hline
    \textbf{Variance} & \text{AU+EU} & \mathbb{V}ar(p^*) &(\mathbb{E}[p^*]-\mathbb{E}[\hat{p}])^2 \\
    \textbf{Differential Entropy} & \text{AU+EU} & h(p^*) &KL(p^*||\hat{p}) \\
    \textbf{QUEST-local}& V_\alpha(p^*) / \big(1 - d_{TV}(p^*_\alpha, \hat p_\alpha)\big) & V_\alpha(p^*) & V_\alpha({q})  \\
    \textbf{QUEST-global} & \int_{\alpha\in[0,1]}V_\alpha(p^*) / \big(1 - d_{TV}(p^*_\alpha, \hat p_\alpha)\big) d\alpha & \int_{\alpha}V_\alpha(p*) d\alpha &\int_{\alpha}V_\alpha({q}) d\alpha    \\ 
    \hline
\end{tabular}
\label{tbl:appx-UMs}
\end{table}
For the calculation of $\alpha$ volume and integrated volume we require the computation of the highest density region of densities $p\in P$ or $q \in \mathcal{Q}$, 1st and 2nd order distributions respectively.
We compute the HDR via a grid approximation\footnote{Note the grid approximation does not suffer a curse of dimensionality issue as we always have 1D first order conditional distributions, and our models are such that we have a 2D second order conditional distribution (i.e. we have two model parameters). For models providing more dimensions techniques such as Monte Carlo sampling make the procedure feasible, but this is outside of the scope of this paper.}, the pseudo-code for which is outlined in \cref{alg:hdr_grid}. 

\begin{algorithm}[ht]
\caption{Grid-based approximation of Highest Density Region (HDR) at level $\alpha$}
\label{alg:hdr_grid}
\begin{algorithmic}[1]
\Require Density function $p(y|x)$ (or $q(\theta|x)$), grid $\mathcal{G} = \{y_i\}_{i=1}^N$ (or $\{\theta_i\}_{i=1}^N$\footnotemark), level $\alpha \in (0,1)$
\Ensure Approximate HDR set $\widehat{\mathcal{C}}_\alpha$

\State Let $\Delta V$ denote the grid cell volume
\State Evaluate densities: $d_i \gets p(y_i|x)$ for all $y_i \in \mathcal{G}$
\State Sort grid points in descending order of density: $d_{(1)} \geq d_{(2)} \geq \cdots \geq d_{(N)}$
\State Initialise cumulative mass: $M \gets 0$
\State Initialise set: $\widehat{\mathcal{C}}_\alpha \gets \emptyset$

\For{$i = 1$ to $N$}
    \State Add point: $\widehat{\mathcal{C}}_\alpha \gets \widehat{\mathcal{C}}_\alpha \cup \{y_{(i)}\}$
    \State Update mass: $M \gets M + d_{(i)} \cdot \Delta V$
    \If{$M \geq 1- \alpha$}
        \State \textbf{break}
    \EndIf
\EndFor

\State \Return $\widehat{\mathcal{C}}_\alpha$
\end{algorithmic}
\end{algorithm}
\footnotetext{We now drop the equivalence with $\theta$. Simply substitute $y$ for $\theta$ in what follows to obtain the algorithm for an HDR of a 2nd order distribution.}

\subsection{Additional results: Without Oracle}
\label{appx:selective-no-oracle}
\Cref{fig:additional-results} shows the sensitivity of QUEST TU measures to knowledge of $p^*$. The advantage seen in \Cref{fig:selective} is reversed, with local QUEST performing comparably to competitors while the former frontrunner, global QUEST, lags behind in unimodal settings (although it retains a slight advantage in bimodal settings). This suggests both that QUEST TU is sensitive to the quality of the model's estimate of $p^*$, and that the properties which make QUEST suited to bimodal data offset some of this sensitivity in settings to which the measure is well suited. 
The interaction between DGP and inferential choices for approximating $p^*$ is an important area for further study. 
% Further investigation is required to definitively conclude which settings and inferential choices under which these effects will be most influential.
\begin{figure}
\centering
\includegraphics[width=\linewidth]{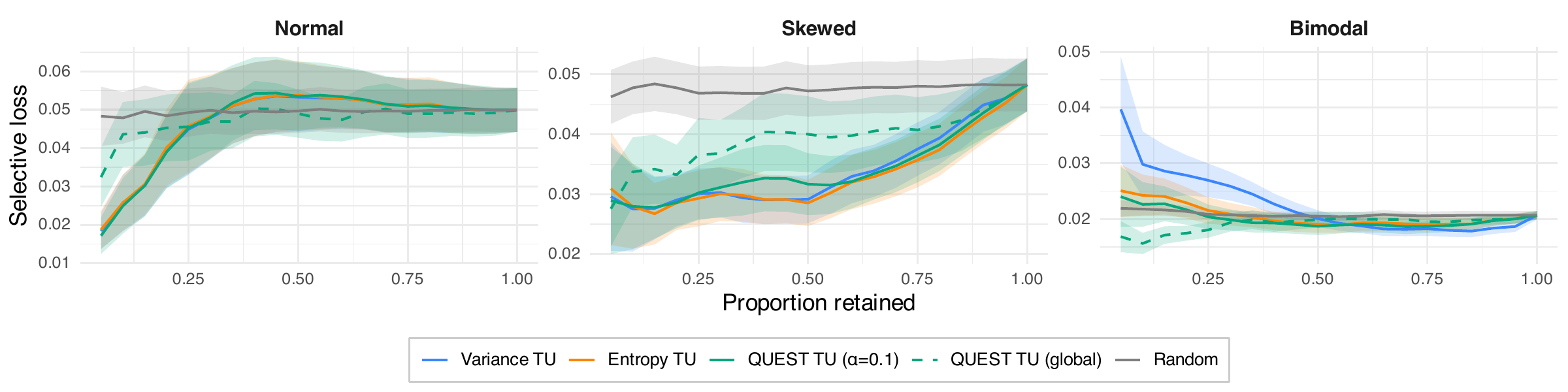} %\includegraphics[width=0.5\linewidth]{Figures-QUEST/selective.pdf}
\caption{Selective learning curves for deep MoE ensembles on three different DGPs. Shading represents standard errors across 20 trials. Instead of the obtaining the oracle density, the model takes the maximum likelihood estimator from among the $M=10$ ensemble members and uses this as a plug-in estimate of $p^*$.}
\label{fig:additional-results}
\end{figure}

\subsection{Compute requirements}
\label{subsec:compute}
Experiments utilised 20 CPU cores on a high performance computer running for a maximum of 24 hours. Maximum memory use was 100GB per repeat. Disk storage requirements were approximately 20GB. No substantial additional compute requirements were need, e.g. for preliminary/failed experiments.

\subsection{Language and code assets}
\label{subsec:appx-python}
We make the experimental code publicly available here \url{https://github.com/sjgoring/modal-uq-public}.
It is written in Python, and we gratefully acknowledge here the use of the following libraries: numpy \citep{harris2020array}, scipy \citep{2020SciPy-NMeth}, joblib \href{https://pypi.org/project/joblib/}{[link]}, matplotlib \citep{Hunter:2007}, scikit-learn \citep{scikit-learn}, torch \cite{DBLP:journals/corr/abs-1912-01703}, cgmm \href{https://cgmm.readthedocs.io/en/latest/}{[link]}.

\section{Supplementary experiment}
We now provide details of, and results from, a supplementary active learning experiment.

\subsection{Active learning}
We use a pool-based procedure with labelled set $\mathcal{D}_L$ and unlabeled pool $\mathcal{D}_U$.
At each acquisition round:
\begin{enumerate}
  \item Train the model on $\mathcal{D}_L$.
  \item Compute acquisition scores $a(x)$ on $x \in \mathcal{D}_U$ using a chosen uncertainty measure.
  \item Select the top-$k$ points by $a(x)$ and add them (with labels) to $\mathcal{D}_L$.
\end{enumerate}
Performance is tracked on a fixed test set after each acquisition round. Unless otherwise stated:
initial labelled size $=200$, batch size $k=200$, and number of rounds $=200$, giving a final pool size $|\mathcal{D}_L|=1200$. 

As with the selective prediction experiment we repeat each experiment 20 times, averaging the reported loss over each of the repeats and reporting standard errors via the shaded regions of the graphs.

\subsection{Data}\label{subsec:mpe_appendix}
We utilise the three synthetic data sets from the selective prediction experiment, and introduce one new data set---the \emph{Multi-Particle Environement}.
We provide an extended description of the modified \verb|simple_adversary| environment. The base environment from PettingZoo's MPE suite \cite{terry2021pettingzoo} is a 2D continuous-control setting in which agents move according to second-order motion dynamics with collisions enabled. The agents have a discrete action space which corresponds to applying a fixed force in a particular cardinal direction, as well as a \textit{no-op} action.

The environment contains one \textit{good} (ego) agent, one or more stationary adversaries acting as obstacles, and a single goal landmark. At reset, the goal is placed within a bounded square region, the ego is placed in an opposing region with a minimum separation from the goal, and obstacles staggered perpendicularly along $\vec{\tau}_{\text{opt}}$. This constructive placement guarantees that the optimal trajectory is geometrically obstructed, requiring the good agent to deviate and resulting in multimodal behaviour suitable for our test cases. An example is shown in \Cref{fig:MPE-example}.

The ego is trained with a Deep Q Network (DQN)~\cite{mnih2015human} policy under a shaped reward which combines a dense goal-seeking term $-\lVert \mathbf{p}_{\text{ego}} -\mathbf{p}_{\text{goal}} \rVert$, a quadratic soft-hinge avoidance term $-w \cdot \max(0, R - d)^2 / R$ where $d$ is the distance to the nearest adversary and $R$ a safety radius, and a goal-reaching bonus when the ego enters a small ball around the goal. The weights are tuned so that the agent prefers both detours rather than collisions and goal-seeking rather than being overly cautious.

For each of $\verb|n_states|$ initial world configurations, we run $\verb|n_traj|$ independent rollouts under fixed $\epsilon$-greedy execution. The input $x$ is the world state including positions and velocities of all entities, flattened. The target $y$ is computed by projecting the ego position at every timestep onto the perpendicular of $\vec{\tau}_{\text{opt}}$, identifying the timestep of maximum absolute deviation, and recording the signed value at that point (left being negative, right being positive). Variance across rollouts from each initial state comes from $\epsilon$-greedy stochasticity acting on a symmetric layout: when the obstacle lies close to $\vec{\tau}_{\text{opt}}$, left and right detours have similar expected return, and small early perturbations push the policy onto one branch. The resulting per-state distribution $p(y \mid x)$ is approximately multimodal.

\begin{figure}[ht]
    \centering
    \includegraphics[width=.7\linewidth]{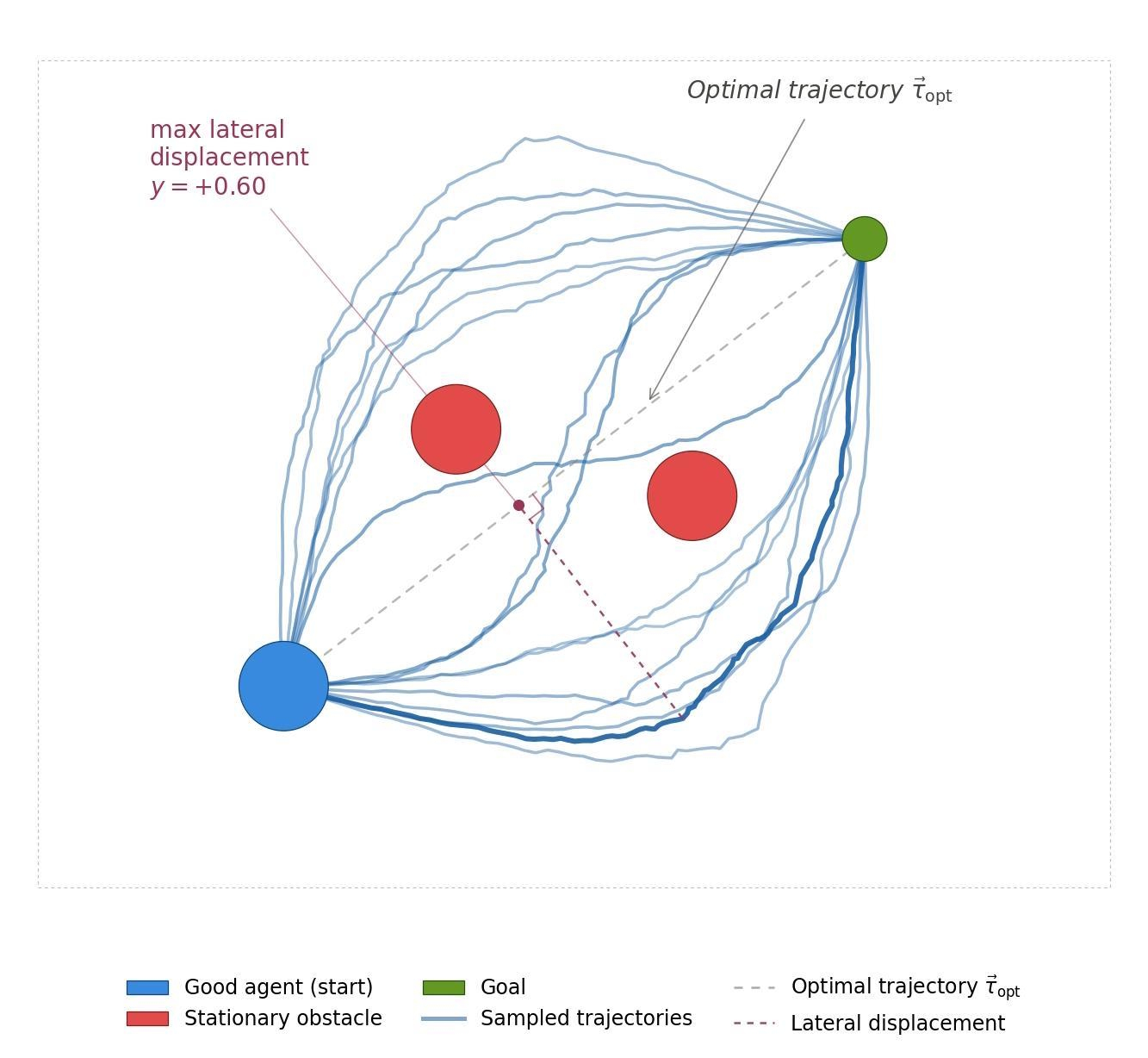}
    \caption{Example MPE initial world configuration with sample trajectories, optimal trajectory and maximum lateral displacement shown.}
    \label{fig:MPE-example}
\end{figure}

\subsection{Model and uncertainty Measures}
We utilise the ensemble mixture of experts model as outlined in \cref{app:moe-description}, and 
utilise the epistemic uncertainty variants of the measure families outlined in \Cref{tbl:appx-UMs}.

\subsection{Evaluation metric}
We evaluate model performance using average absolute distance between the predicted and true mode over points in the test set $x_{test}$, that is:
\begin{equation*}
    \mathcal{L}(\hat{y}|x) = \sum_{x\in x_{test}}\frac{1}{|x_{test}|}|(y^*|x)-(\hat{y}|x)|
\end{equation*}
where as in \Cref{sec:appendix_experiments} $y^{*}=\arg\max_{y},p^{*}(y\mid x), \hat{y}=\arg\max_{y},\hat{p}(y\mid x)$, and $p^*$ is the true density and $\hat{p}$ the learnt density. 
% \subsection{Compute requirements, language and code assets}
See \Cref{subsec:compute,subsec:appx-python} for compute requirements and code assets.

\subsection{Results and interpretation}
\begin{figure}[ht]
    \centering
\begin{tabular}{l l}
    \includegraphics[width=0.45\linewidth,trim={0 0 0 1cm},clip]{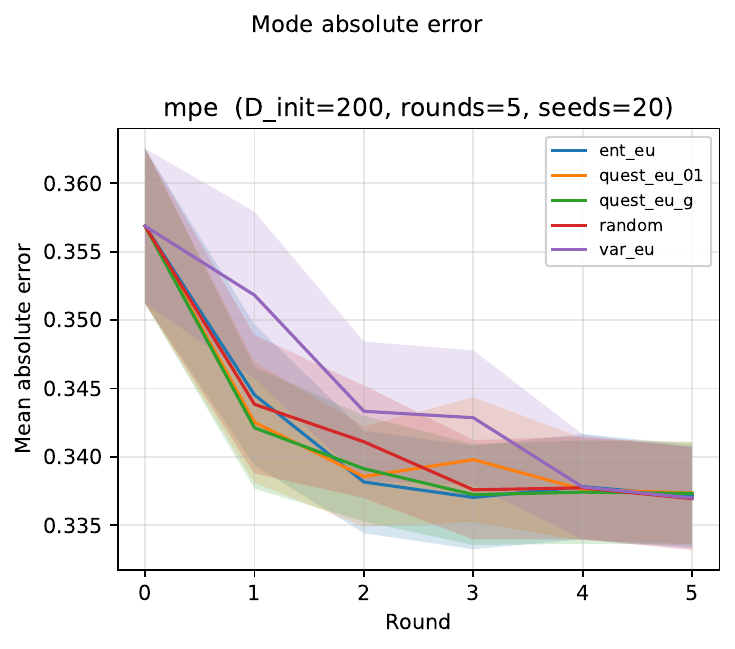} 
    &  \includegraphics[width=0.45\linewidth,trim={0.6cm 0 0 1cm},clip]{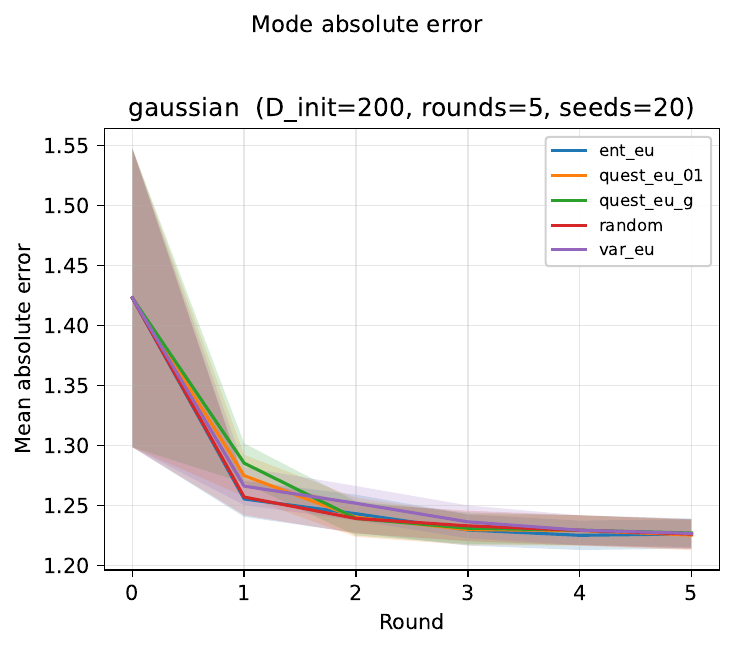}\\
     \includegraphics[width=0.45\linewidth,trim={12.5cm 0 0 1cm},clip]{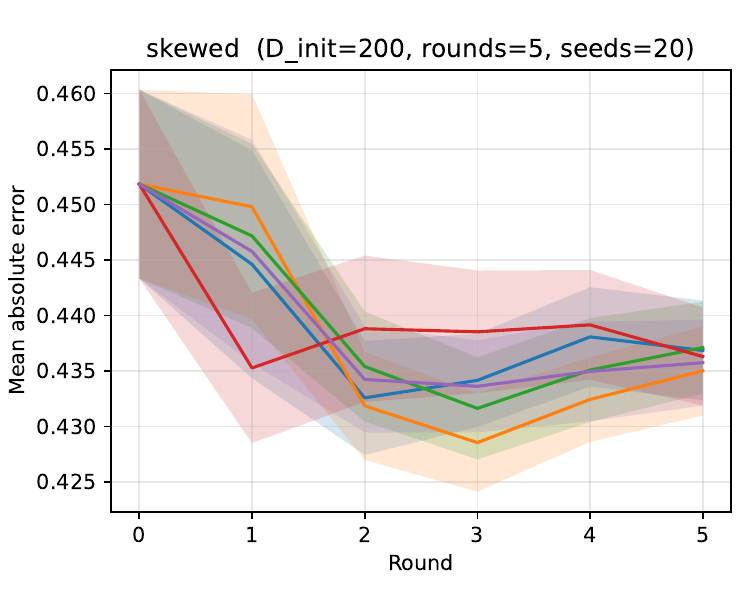}
     & \includegraphics[width=0.45\linewidth,trim={13.1cm 0 0 1cm},clip]{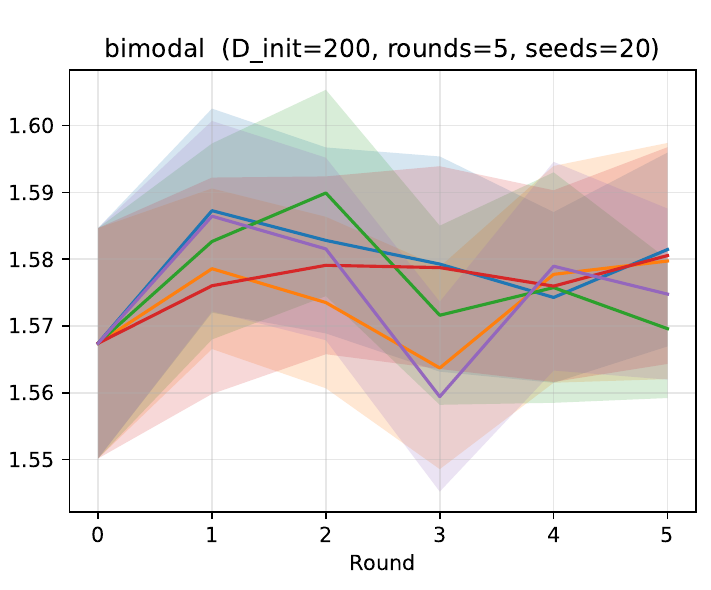}
\end{tabular}
    \caption{Active learning experiment with ensemble MoE model across four datasets. Pool labelling conducted by EU measures from the variance, entropy, QUEST-local and QUEST-global families. All measures are compared to a random baseline.}
    \label{fig:AL}
\end{figure}
\Cref{fig:AL} shows QUEST performance inline with competitor measures on these 4 data sets. 
As in the selective prediction experiment, the Oracle density is utilised.
This ensures the best possible uncertainty measurement (w.r.t faithfulness to the true value) for both the variance and differential entropy measures of EU. 
The Oracle is of no use to QUEST--which we might think would put it at a disadvantage--the measure of 2nd order HDR volume is wholly subjective. 
Yet despite this, both QUEST measures hold their own and perform equally well whilst their competitors are maximally boosted by the presence of the Oracle.
A line of further enquiry is to see the effect of various inferential choices (see \citet{schweighofer2025, fishkov_uncertainty_2025}) required to estimate the Oracle distribution on the performance of variance/entropy.
We suspect this might be where QUEST begins to exhibit superior relative performance.

\section{Admissible measures of $TU$}\label{app:admissible}

While total variation distance offers a particularly transparent choice---bounded by $1$ and easily interpreted as the maximum probability disagreement on any event---several other $f$-divergences satisfy the admissibility conditions of \Cref{sec:method} and yield equally valid $TU$ constructions. We sketch three alternatives here.

The \textit{squared Hellinger distance} uses $f(t) = (\sqrt{t} - 1)^2$, with $f(0) = 1$ and $f'(\infty) = 1$, both finite. The resulting divergence $H^2(p, q) = \int (\sqrt{p} - \sqrt{q})^2\, dx$ is bounded above by $M_f = 2$, with the maximum attained for distributions with disjoint support. The natural penalty function is $g_f(D) = 1/(1 - D/2)$, giving:
$$\mathrm{TU}_{\mathrm H} := V_\alpha(p_{\theta^*}) \cdot \frac{1}{1 - H^2(p^*_\alpha, \hat{p}_\alpha)/2}.$$

The \textit{Jensen-Shannon divergence} is a symmetrized version of KL: $\mathrm{JSD}(p, q) = \frac{1}{2}D_{\mathrm{KL}}(p \| m) + \frac{1}{2}D_{\mathrm{KL}}(q \| m)$ where $m = (p+q)/2$. It corresponds to $f(t) = \frac{1}{2}[t \log t - (t+1)\log\frac{t+1}{2}]$, with $f(0) = f'(\infty) = \log 2 / 2$, giving $M_f = \log 2$. The penalty $g_f(D) = 1/(1 - D/\log 2)$ yields:
$$\mathrm{TU}_{\mathrm{JS}} := V_\alpha(p_{\theta^*}) \cdot \frac{1}{1 - \mathrm{JSD}(p^*_\alpha, \hat{p}_\alpha)/\log 2}.$$

The \textit{triangular discrimination} (aka Le Cam distance) uses $f(t) = (t-1)^2/(t+1)$, with $f(0) = 1$ and $f'(\infty) = 1$, giving $M_f = 2$. The penalty $g_f(D) = 1/(1-D/2)$ yields:
$$\mathrm{TU}_\Delta := V_\alpha(p_{\theta^*}) \cdot \frac{1}{1 - \Delta(p^*_\alpha, \hat{p}_\alpha)/2},$$
where $\Delta(p, q) = \int (p-q)^2/(p+q)\, dx$.

In practice, the choice between these constructions reflects a modeling decision about which kind of calibration error is most important to detect. We default to the TVD-based measure in the main text for its interpretive simplicity and its historical primacy in statistical distance literatures.

{
\section{Further remarks on levelling} \label{sec:app-levelling}

Levelling represents a principled and general concept of spread with several notable advantages. First, it applies to distributions of bounded or unbounded support, as well as distribution pairs with identical or distinct (potentially even disjoint) support. Second, it preserves the intuition that a Dirac mass and the uniform distribution represent natural limits. To see this, observe that levelling can never produce a Dirac mass, whereas for fixed support $\mathcal X$, a uniform $p$ is maximally levelled. (Both properties follow from $\mathrm{L}1$.(c).)
Finally, because levelling does not rely on a distribution's moments, it is well-defined even for densities with infinite variance, such as the Cauchy.

To see this, consider the examples in \Cref{fig:levelling}(A), which differ only in their scale parameter $\gamma$.
Levelling captures the intuitive ordering of these densities by their concentration. 
For any pair $p_\gamma$ and $p_{\gamma'}$ with $\gamma < \gamma'$, condition $\mathrm{L}1$ holds with $A$ the central region where $p_\gamma$ exceeds $p_{\gamma'}$, and $B$ the tails where the relationship reverses. The two densities cross at the unique points where their values coincide; on $A$, $p_\gamma$'s mass is excess relative to $p_{\gamma'}$, and this mass is exactly redistributed to the tails $B$ in $p_{\gamma'}$, satisfying $\mathrm{L}1$'s mass-conservation and level-set conditions. 
Variance fails to make these comparisons because it is undefined for all three densities, while levelling correctly identifies the ordering $p_{1/2} \prec p_1 \prec p_2$.

\begin{figure}[t]
    \centering
    \includegraphics[width=0.9\linewidth]{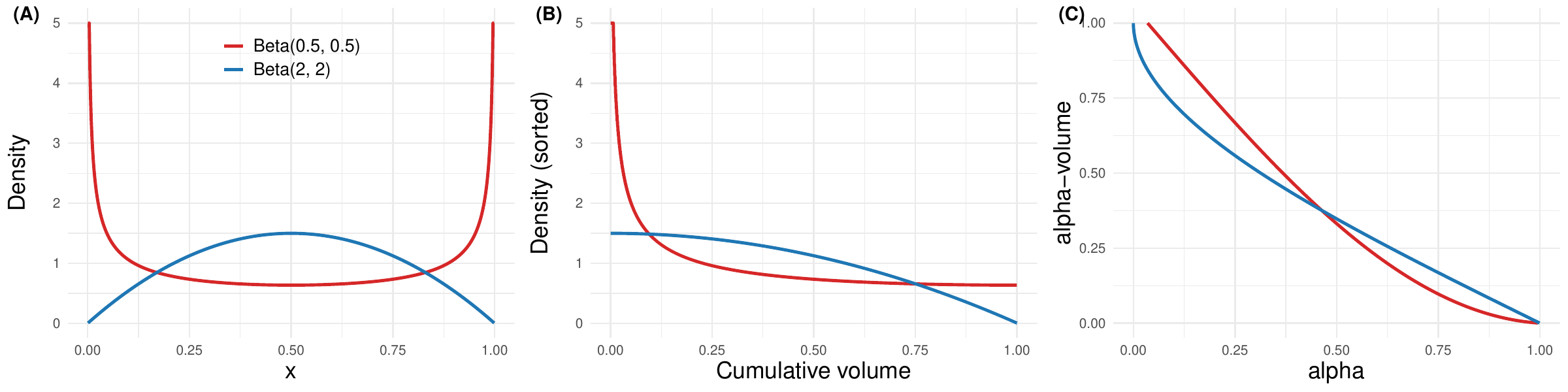}
    \caption{$\mathrm{Beta}(0.5, 0.5)$ and $\mathrm{Beta}(2, 2)$ are incomparable w.r.t. levelling. \textbf{(A)} Density curves for both distributions. \textbf{(B)} Decreasing rearrangements cross. \textbf{(C)} Coverage curves therefore cross as well.}
    \label{fig:betas}
\end{figure}

We remark that levelling imposes a \textit{partial} order on distributions, as there exist some pdf pairs such that neither is a levelling of the other. 
This occurs when their decreasing rearrangements (and therefore coverage curves) cross, as in the case of $\mathrm{Beta}(0.5, 0.5)$ and $\mathrm{Beta}(2, 2)$ (see \Cref{fig:betas}).
This means that levelling is not \textit{strictly} more general than variance, as there are cases (such as this one) where variance gives a partial order while levelling abstains. In particular, we have $\mathrm{Var}\big(\mathrm{Beta}(2, 2)\big) < \mathrm{Var}\big(\mathrm{Beta}(0.5, 0.5)\big)$. 
Note that these distributions are still comparable w.r.t. local QUEST measures at any \textit{fixed} $\alpha$---the incomparability arises because our ordering flips as $\alpha$ varies (see \Cref{fig:betas}(C)). 
As for the global measure, we have $IV(\mathrm{Beta}(1/2, 1/2)) = 1 - 2/\pi \approx 0.363$ and $IV(\mathrm{Beta}(2, 2)) = 0.375$. Thus $IV$ and variance give opposing results in this case.

Perhaps more interestingly, we can construct examples where \textit{levelling} and variance render contradictory verdicts. This occurs, for instance, when $p$ is a bimodal distribution with sharp peaks spaced far apart, while $p'$ redistributes mass from those peaks to the centre, subject to $\mathrm{L}1$ constraints. This will tend to \textit{reduce} variance, despite the fact that $p \prec p'$. 
Appealing to the intuition that ``spread'' should track ``distance-to-uniformity'', we argue that variance renders the wrong judgment in this case.

\section{Further remarks on uncertainty axioms}\label{app:axioms}

\paragraph{Alternative axioms} Although our axiom $\mathrm{A}1$ accords with that of \citet{wimmer2023} and \citet{sale_second-order-dist_2023}, QUEST measures do not satisfy a stronger version proposed by \citet{bulte_axiomatic_2025}:
\begin{itemize}[noitemsep]
    \item[$\mathrm{A}1^*$:] $EU(q)=0$ iff $q = \delta_\theta$.
\end{itemize}
Specifically, we respect the reverse but not the forward direction of this biconditional.
The latter fails, for instance, on Dirac mixtures or joint distributions with support on a submanifold. 
These are cases of measure zero sets that do not correspond to any single Dirac mass. 
We can generalize the point by proposing our own modified biconditional:
\begin{itemize}[noitemsep]
    \item[$\mathrm{A}1'$:] $EU(q)=0$ iff $\lambda\big( \mathrm{supp}(q) \big)=0$,
\end{itemize}
which QUEST measures do satisfy. (Entropy and variance, however, do not.) 
While $\mathrm{A}1'$ may be a reasonable axiom to adopt, we do not develop the point any further here.

Another proposed axiom is due to \citet{sale_second-order-dist_2023}:
\begin{itemize}
    \item[$\mathrm{A}6$:] For all $p,q$, (i) $AU(p) \leq TU(p, q)$ and (ii) $EU(q) \leq TU(p, q)$.
\end{itemize}
QUEST measures satisfy $\mathrm{A}6$.(i) but not $\mathrm{A}6$.(ii), since a flat second-order distribution may result in a well-calibrated predictor. However, the inequality is arguably nonsensical in our setting, since the $k$-dimensional Lebesgue measure on the lhs is not obviously comparable to the $d$-dimensional Lebesgue measure on the right. Thus this axiom is inappropriate for our framework.

\paragraph{Should $EU$ be purely subjective?} Given that $TU$ decreases as model calibration improves---an undisputed premise, as far as we are aware---all UQ methods face a choice. 
Either (1) decide that $TU$ is fully determined by $AU$ and $EU$, in which case the latter is necessarily truth-relative; or (2) embrace a subjective notion of $EU$, at the cost of complicating the classical uncertainty typology. 
We argue that (1) conflates calibration and confidence. Agents who are confidently miscalibrated and confidently well-calibrated experience similar subjective uncertainty (e.g., agents $A, B$ with $q_A = \delta_\theta, q_B = \delta_{\theta^*}$ for some $\theta \neq \theta^*$). 
% Similarly, agents can enjoy comparable calibration with differing levels of confidence (e.g., if agent $A$ is confident about $\mu$ but uncertain about $\sigma$, and $B$ is confident about $\sigma$ but uncertain about $\mu$, they could produce the same posterior predictive). 
UQ frameworks that choose (1) over (2) have no way of properly diagnosing such cases. 

A plausible objection runs as follows. Whatever the functional form of $TU$, it should surely be monotone in $EU$. Since $AU$ is fixed by the DGP, this means that $EU$ should track some notion of error or miscalibration. On a purely subjective measure, an irrational agent may fail to update $q$ properly given evidence, leading to decreased $EU$ but increased $TU$.
This is nonsensical. Call this \textit{the monotonicity objection}.

The monotonicity objection implicitly assumes something like the additive decomposition of $TU$, which holds for variance- and entropy-based measures. Whatever the merits of this view---which has been challenged by recent works \citep{marshall2011inequalities, wimmer2023}---it should not be assumed upfront, but rather derived from first principles. 
To do otherwise when attempting to formalize UQ would be to beg the question. Interestingly, our framework accords with the intuition of the monotonicity objection when $\theta^* \in \Theta$ and the agent has a consistent learning procedure. In this setting, $q$ will tend to concentrate around true parameter values as data accumulates, ensuring a monotone relationship between epistemic and total uncertainty. 
We point out that model misspecification and/or irrational belief updates are genuine problems, but hardly unique to our framework---these would cause headaches for any effort to model uncertainty.

So the objection loses force when a single agent learns over time. But what about a comparison between two separate agents, one of whom is more confident in their predictions than the other? On QUEST measures, the more confident agent by definition has lower $EU$, even if their predictions are wildly miscalibrated. We argue that this is exactly what a good uncertainty measure should do, otherwise we have no language for disambiguating what is the same and what is different between agents who are equally sure of their models, despite vastly different empirical performance. We locate the similarity in $EU$ and the difference in $TU$, which directly rewards the better calibrated model. The trade-off is that we lose the monotonicity property in general---not in the case of a single consistent learner, but possibly when comparing multiple models. This is why $TU$ fails to satisfy $\mathrm{A}4$.

As a final comment, we anticipate a question as to why we define $EU$ in terms of $q$ rather than the posterior predictive $\hat p$. The answer is simple. We may be confident that $\sigma$ is very large (low subjective uncertainty, high $V_\alpha(\hat p)$) or very uncertain about $\mu$ but with fixed small $\sigma$ (high subjective uncertainty, low $V_\alpha(\hat p)$).
The true focus of a subjective uncertainty measure is therefore the second-order distribution $q$, not the posterior predictive.

% \newpage
% \input{checklist}
\end{document}